\def\algorithmSize{\footnotesize}

\newif\ifwrapHLExecAlgo
\newif\ifwrapHLLearnAlgo
\newif\ifwrapBilevelExecAlgo

\wrapHLExecAlgotrue
\wrapHLLearnAlgotrue
\wrapBilevelExecAlgotrue

\documentclass{article}
\usepackage[preprint]{neurips_2026}
\usepackage[utf8]{inputenc} % allow utf-8 input
\usepackage[T1]{fontenc}    % use 8-bit T1 fonts
\usepackage{hyperref}       % hyperlinks
\usepackage{url}            % simple URL typesetting
\usepackage{booktabs}       % professional-quality tables
\usepackage{amsfonts}       % blackboard math symbols
\usepackage{nicefrac}       % compact symbols for 1/2, etc.
\usepackage{microtype}      % microtypography
\usepackage[table]{xcolor}  % colors
\usepackage{wrapfig}        % text wrap around figures/tables
\usepackage{enumitem}       % custom list formatting
\usepackage{dillon}         % Dillon's packages and macros
\title{
Learning Bilevel Policies over Symbolic World Models for Long-Horizon Planning
}

\author{
    Dillon Z. Chen$^{1,2,3,4}$ \quad Till Hofmann$^{5}$ \quad Toryn Q. Klassen$^{1,2}$ \quad Sheila A. McIlraith$^{1,2}$
    \\ \\
    $^{1}$Vector Institute \quad
    $^{2}$University of Toronto
    \\
    $^{3}$LAAS-CNRS \quad
    $^{4}$University of Toulouse \quad
    $^{5}$RWTH Aachen University
}

\begin{document}

\maketitle
%%%%%%%%%%%%%%%%%%%%%%%%%%%%%%%%%%%%%%%%%%%%%%%%%%%%%%%%%%%%%%%%%%%%%%%%%%%%%%

\begin{abstract}
We tackle the challenge of building embodied AI agents that can reliably solve long-horizon planning problems.
Imitation learning from demonstrations has shown itself to be effective in training robots to solve a diversity
of complex tasks requiring fine motor control and manipulation over low-level (LL), continuous environments.
Yet, it remains a difficult endeavour to generate long-horizon plans from imitation learning alone.
In contrast, high-level (HL), symbolic abstractions facilitate efficient and interpretable long-horizon planning.
We propose to combine the strengths of LL imitation learning for manipulation and control, and HL symbolic abstractions for long-horizon planning.
We realise this idea via \emph{bilevel policies} of the form $(\pi^{\mathrm{hl}}, \pi^{\mathrm{ll}})$, consisting of a neural policy $\pi^{\mathrm{ll}}$ learned from LL demonstrations, and an HL symbolic policy $\pi^{\mathrm{hl}}$ that is constructed from symbolic abstractions of the LL demonstrations combined with inductive generalisation.
We implement these ideas in the BISON system.
Experiments on extended MetaWorld benchmarks demonstrate that BISON generalises to long horizons and problems with greater numbers of objects than those solved by VLA and end-to-end methods, 
and is more time and memory efficient in training and inference.
Notably, when ignoring LL execution, BISON's HL policies can solve HL problems with 10,000 relevant objects in under a minute.
\begin{center}
    Project page: \url{https://dillonzchen.github.io/bison}
\end{center}
\end{abstract}

\begin{figure}[ht!]
    \centering
    \input{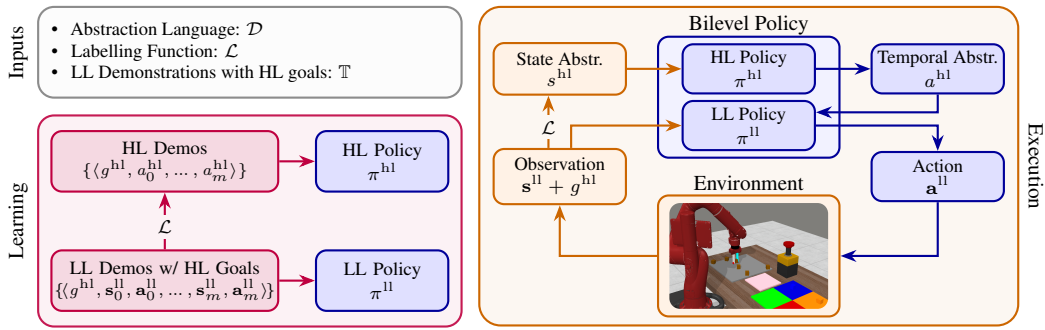}
    \caption{
        \textbf{Top Left} -- inputs for learning and executing bilevel policies: a domain theory $\domain$, a labelling function $\labelling$ that maps observations to state abstractions, and LL demos with HL goals.
        \textbf{Bottom Left} -- bilevel policy learning:
        LL demos induce HL demos via $\labelling$, and LL/HL policies are separately learned from LL/HL demos.
        \textbf{Right} -- bilevel policy execution:
        state abstractions $\hls$ are computed from observations $\lls$ via $\labelling$ to propose HL actions $\hla$ which in turn help propose LL actions $\lla$.
    }\label{fig:pipeline}
\end{figure}
\section{Introduction}
Long-horizon planning remains a longstanding challenge to building embodied AI agents.
We present a framework that combines the strengths of symbolic- and neural-based learning and reasoning to produce goal-conditioned policies that address long-horizon planning and acting in continuous state and action spaces with varying numbers of objects.

Scaling up reinforcement learning in simulation and imitation learning from demonstrations has produced robot policies capable of diverse motor control and manipulation tasks \citep{ahn.etal.2022,hoffmann.etal.2022,wei.etal.2022,huang.etal.2022,driess.etal.2023,hu.etal.2023,zitkovich.etal.2023,kim.etal.2024,intelligence.etal.2025,team.deepmind.2025,intelligence.etal.2026,zhang.etal.2026}.
However, scaling alone currently appears insufficient for efficient long-horizon planning and acting \citep{dziri.etal.2023,kambhampati.etal.2024,lin.etal.2025,park.etal.2025}.
In contrast, symbolic planners (e.g., \citep{helmert.2006,correa.etal.2020,scala.etal.2020,seipp.etal.2020,bonassi.etal.2025,speck.etal.2025}) over symbolic world models have proven to be highly effective long-horizon planners, but do not easily extend to settings with varying sensory modalities and hard-to-model low-level dynamics. 
More recently, a suite of bilevel planning systems has emerged that combine symbolic and neural methods (e.g., \citep{srivastava.etal.2014,konidaris.etal.2018,garrett.etal.2021,shen.etal.2026}). 
Unfortunately, they reflect a number of the inherent shortcomings of symbolic planning as they can struggle with open-world planning and partial observability.

In this work, as depicted in \Cref{fig:pipeline}, we present a novel approach to long-horizon planning, combining the benefits of low-level (LL) imitation learning with the affordances of high-level (HL), symbolic abstraction, reasoning, and learning.
Rather than generating plans from a symbolic domain theory, our approach extracts and inductively generalises symbolic policies from abstractions of our LL demonstrations, and learns to realise those symbolic policies at the LL by means of behaviour cloning.

We realise our approach via the creation of \emph{bilevel policies} of the form $(\hlp, \llp)$, consisting of a neural policy $\llp$, learned from multiple LL demonstrations, and an HL symbolic policy $\hlp$ that is constructed from a symbolic abstraction of the LL demonstrations. 
We apply goal regression \citep{waldinger.1977,lozano.etal.1984,reiter.1991,xu.etal.2019}---a pre-image rewriting technique---to the abstracted symbolic demonstrations to extract a set of \emph{Condition} $\mapsto$ \emph{Action} rules. 
We inductively generalise the resulting rules to produce compact and expressive HL policies consisting of first-order, condition-action rules that can operate in open-world environments \citep{reiter.2001,sanner.boutilier.2009,liu.etal.2025}.
Our LL policies $\llp$ are represented by a graph neural network \citep{scarselli.etal.2009,kipf.welling.2017,gilmer.etal.2017} conditioned on actions returned by $\hlp$.

We implement our approach in the \genplan{} system.
We compare \genplan{} to 8 baselines spanning VLA, end-to-end, and symbolic planning methods on 8 environments extending the MetaWorld benchmarks~\citep{yu.etal.2019,mclean.etal.2025}.
Experiments show that \genplan{} solves tasks with longer horizons and more objects than those solved by end-to-end and VLA baselines, and handles more complex environment dynamics than symbolic planning baselines.
Our contributions include:

\begin{itemize}
    \item We introduce \genplan{}, an approach that learns \emph{bilevel policies} $(\hlp, \llp)$ from LL demonstrations paired with HL goals, given a domain theory $\domain$ and labelling function $\labelling$.
    \item We apply goal regression and inductive generalisation to learn compact, expressive HL policies $\hlp$ that can generalise to arbitrary numbers of objects, and can operate in open-world and partially observable environments, and we learn LL policies via a compact graph neural network $\llp$ with fewer than 33,000 parameters.
    \item We conduct 21,600 episodes' worth of experiments across various baselines and environments extending the MetaWorld benchmark. 
    Results demonstrate that \genplan{} achieves superior efficiency, generalisation, and long-horizon planning capabilities compared to baselines.
    When ignoring LL execution, \genplan{}'s HL policies can solve problems with at least 10,000 objects in under a minute.
\end{itemize}

% !TEX root =mainfile.tex

\section{Related Work}\label{sec:related}

\paragraph{Planning with Abstractions and Hierarchical Reinforcement Learning} 
The idea of leveraging abstractions to facilitate efficient long-horizon planning is a longstanding one \citep{sacerdoti.1974,tate.1977,giunchiglia.walsh.1992,bacchus.yang.1994,boutilier.dearden.1994,erol.etal.1996,dean.givan.1997,nau.etal.2003,marthi.etal.2008,konidaris.2019,bercher.etal.2019}.
Reinforcement learning (RL) approaches have also looked at learning and leveraging hierarchies and abstractions \citep{dayan.hinton.1992,hauskrecht.etal.1998,sutton.etal.1999,dietterich.2000,li.etal.2006,konidaris.etal.2018,abel.etal.2016,abel.etal.2018,le.etal.2018} to improve sample efficiency.
There has also been recent interest in leveraging formal languages to represent such abstractions as well as to specify goals for RL agents
\citep{li.etal.2017,icarte.etal.2018,camacho.etal.2019,bozkurt.etal.2020,illanes.etal.2020,vaezipoor.etal.2021,icarte.etal.2022,voloshin.etal.2022,qiu.etal.2023,yalcinkaya.etal.2024,jackermeier.abate.2025,li.etal.2025a}.
Our work is inspired by the use of HL abstractions but differs in the problem setting:
we learn policies entirely from demonstrations over both the abstraction and underlying environment that generalise to longer horizon problems than RL and exploration methods.

\paragraph{Task and Motion Planning}
Task and motion planning (TAMP) is a common framework for solving embodied AI that integrates HL planning and LL execution.
TAMP methods can be categorised \citep{garrett.etal.2021,zhao.etal.2025} into interleaved search-then-sample \citep{srivastava.etal.2014,shah.etal.2020,mendezmendez.etal.2023,shen.etal.2025} or hybrid constraint satisfaction and optimisation \citep{lozanoperez.kaelbling.2014,toussaint.2015,dantam.etal.2018} approaches.
Later works employ learning in various manners: predicting HL action feasibility \citep{wells.etal.2019,driess.etal.2020,xu.etal.2022,bouhsain.etal.2023,yang.etal.2023,bouhsain.etal.2025}, learning problem transformations, heuristics or policies for guiding search \citep{chitnis.etal.2016,kim.etal.2019,silver.etal.2021a,curtis.etal.2022,khodeir.etal.2023,mandlekar.etal.2023,mendezmendez.etal.2023,cieslar.etal.2024,du.etal.2026}, and learning policies that aim to imitate bilevel planners \citep{driess.etal.2020a,mcdonald.hadfieldmenell.2021,zhu.etal.2021,lin.etal.2022}.
Our work falls in the latter category whereby we learn policies that solve bilevel planning problems without relying on search.
Our work differentiates itself within this category through learning bilevel policies over both the LL and HL spaces for facilitating efficient, interpretable and long-horizon planning in contrast to purely end-to-end methods.
Furthermore, some TAMP search methods struggle with open-world planning and partial observability, owing to their reliance on symbolic planners, while our approach can support these problem settings.

% !TEX root =mainfile.tex

\section{Problem Statement}\label{sec:background}
We are interested in learning policies over continuous environments that address HL goals from (1) LL demonstrations paired with HL goals, (2) an abstraction language $\domain$, and (3) a labelling function $\labelling$ that maps LL observations to HL abstractions, as depicted in \Cref{fig:pipeline},
that can generalise to unseen states and goals in a zero-shot manner.
We follow a basic \emph{bilevel planning} paradigm~(e.g., \citep{srivastava.etal.2014,garrett.etal.2021}) in which a planning problem, consisting of continuous LL states and actions, admits an HL model-based abstraction using symbolic and relational states and actions.
We then list assumptions applied in our planning representations, and corresponding related work that validates such assumptions.
A summary of definitions, notations and typography semantics is provided in Appendix \ref{app:notation}.

\paragraph{Bilevel Planning} A \emph{planning problem} is a tuple $\problem = \problemDefn$ where $\states$ is a set of states, $\actions$ is a set of actions, $\init$ is the initial state and $\goal \subseteq \states$ is a set of goal states.
An action $\action \in \actions$ is a function $\action : \states \to \Delta(\states \cup \set{\bot})$ that maps to a distribution of possible next states or $\bot$ indicating the action is not applicable.
A policy is a distribution over actions $\pi(\action \mid \state, \ast)$ conditioned on a state and other arguments $\ast$.
A policy is a \emph{solution} for a planning problem if it achieves a goal state from the initial state $\init$ with an expected probability of 1.
A \emph{bilevel planning problem} is a tuple $\bilevelProb=\bilevelDefn$ consisting of an HL and LL planning problem and a \emph{labelling function} $\labelling: \states^{\ll} \to \states^{\hl}$ such that $\hlg = \seta{\labelling(\lls) \mid \lls \in \llg}$.
We next depict conditions on $\labelling$ for ensuring useful HL abstractions of LL problems, followed by representations of HL and LL planning problems for our work.

\paragraph{Downward Refinement Property}
We next introduce a property of our bilevel planning paradigm that establishes when the HL abstraction is useful for constructing bilevel planning solutions.
The downward refinement property (DRP) \citep{bacchus.yang.1994} states that any solution for an HL problem can be refined into a solution for the LL problem.
We extend this notion to the general nondeterministic setting.
Given a state $\state$ and policy $\pi$, we denote $\exec(\state, \pi)$ to be the set of all possible successor states achievable by applying some action with non-zero support in $\pi(\cdot \mid \state)$. 
\begin{definition}[Nondeterministic Downward Refinement Property]\label{def:ndrp}
    A bilevel planning problem $\bilevelDefn$ exhibits \emph{the nondeterministic downward property (NDRP)} if for every $\hlp$ that solves $\hlprob$, there exists $\llp$ that solves $\llprob$ where for every $\lls$ reachable by $\llp$ in $\llprob$,
    \begin{align}
        \forall \state \in \exec(\lls, \llp), \labelling(\state) \in \exec(\labelling(\lls), \hlp) \cup \seta{\labelling(\lls)}.
    \end{align}%
\end{definition}%
Intuitively, the LL policy can be viewed as analogous to a set of skills or options \citep{konidaris.barto.2009} for implementing and refining HL actions. 
NDRP states that every step determined by $\llp$ either stays within the same HL abstract state or advances to a valid HL successor under $\hlp$.
In this way, a solution for the HL problem can be used as guidance to build a solution for the underlying LL problem.

\paragraph{LL Representation: Object-Centric, Ego-Centric Viewpoints}
We represent states in our LL problems $\llprob$ in an object- and ego-centric manner.
This means that states are represented as tuples $\lls = \gena{\x_e, \set{\x_o}_{o \in \objects}}$ where $\x_e \in \R^n$ is a vector representing the agent's state (e.g., location of a robot), $\objects$ is a set of objects visible to the agent, and $\set{\x_o}_{o \in \objects}$ where $\x_o \in \R^m$ is a set of vectors representing the states (e.g. position) of all other objects in the world relative to the agent.
LL actions are black box functions that determine transitions between LL states with unknown probabilities.
The object-centric representation is commonly used for embodied AI planning~\citep{silver.etal.2022,silver.etal.2023,kumar.etal.2024,liang.etal.2025,shen.etal.2026} for which there is much work for learning or extracting such representations from raw sensory data~\citep{burgess.etal.2019,anand.etal.2019,kipf.etal.2020,zadaianchuk.etal.2021,james.etal.2022,yuan.etal.2023,haramati.etal.2024,wen.etal.2025,park.etal.2026}.

\paragraph{HL Representation: Relational State and Action Abstractions}
We represent our HL problems $\hlprob$ via state and temporal abstractions of the world in a relational form, inspired by the STRIPS~\citep{fikes.nilsson.1971} and PDDL languages~\citep{mcdermott.etal.1998,ghallab.etal.2004,geffner.bonet.2013,haslum.etal.2019} from symbolic planning.
Specifically, LL states can be abstracted into HL states represented as symbolic databases and LL actions can be temporally abstracted into HL actions via an action schema language.
We define a planning \emph{domain} as a tuple $\domain = \langle \predicates, \schemata \rangle$ consisting of a set of predicates $\predicates$ and action schemata $\schemata$.

A \emph{predicate} is a symbol $p$ with a set of argument terms, denoted $p(x_1, \ldots, x_{n_p})$ where $n_p \in \N_0$ depends on $p$ and denotes the arity of $p$.
A fact is a predicate whose argument terms are all instantiated with objects denoted $p(o_1, \ldots, o_{n_p})$, and an HL state is a set of facts.
We concern ourselves with \emph{open-world planning} where facts that are not present in the state are assumed unknown \citep{green.1969}, in contrast to the more restrictive closed-world planning exhibited by PDDL where facts not present in the state are assumed false \citep{fikes.nilsson.1971}.
Our labelling function $\labelling$ maps from LL to HL states with predicates from $\domain$.
We assume that all LL goal states can be abstracted into a single HL \emph{goal condition} $\hlg$ defined as a set of facts.
An HL state $\hls$ is a goal state if $\hlg \subseteq \hls$.

An \emph{action schema} is a temporally-abstracted, nondeterministic action that can be parameterised by objects, given by a tuple $\gena{ \var(a), \pre(a), \seta{\add_i(a), \del_i(a)}_{i=1}^{n} }$ where $\var(a)$ is a set of parameter variables, and the preconditions $\pre(a)$ and nondeterministic add $\add_i(a)$ and delete $\del_i(a)$ effects are sets of predicates with arguments instantiated with variables from $\var(a)$.
An HL action $\hla$ is an action schema $a$ where each variable is instantiated with an object, denoted $\hla = a(o_1, \ldots, o_{n_a})$, and is applicable in an HL state $\hls$ if $\pre(\hla) \subseteq \hls$ in which case we define its successor states 
\begin{align}
    \succ(\hls, \hla) = \seta{(\hls \setminus \del_i(\hla)) \cup \add_i(\hla)}_{i=1}^n. \label{eq:succ}
\end{align}

\begin{example}[Pick and Place Domain]\label{ex:pick-place}
  We can model a simple pick and place domain $\domain = \gen{\predicates, \schemata}$ where a robot can transport items between locations with
  $\predicates = \{ \mi{rAt}(\mi{x}), \mi{at}(\mi{x}, \mi{y}), \allowbreak \mi{free}(), \allowbreak \mi{hold}(\mi{x}) \}$ and $\schemata$ containing the three deterministic action schemata:
  \begin{center}
    \scriptsize
    \begin{tabularx}{\textwidth}{l l l}
      $\begin{aligned}
        \mi{pick(o, l)}:
        \pre &: \{ \mi{rAt}(\mi{l}), \mi{at}(\mi{o}, \mi{l}), \mi{free}()\} \\
        \add &: \{ \mi{hold}(\mi{o})\} \\
        \del &: \{ \mi{at}(\mi{o}, \mi{l}), \mi{free}()\}
      \end{aligned}$
      &
      $\begin{aligned}
        \mi{move(l1, l2)}:
        \pre &: \{ \mi{rAt}(\mi{l1})\} \\
        \add &: \{ \mi{rAt}(\mi{l2})\} \\
        \del &: \{ \mi{rAt}(\mi{l1})\}
      \end{aligned}$
      &
      $\begin{aligned}
        \mi{place(o, l)}: 
        \pre &: \{ \mi{rAt}(\mi{l}), \mi{hold}(\mi{o})\} \\
        \add &: \{ \mi{at}(\mi{o}, \mi{l}), \mi{free}()\} \\
        \del &: \{ \mi{hold}(\mi{o})\}
      \end{aligned}$
    \end{tabularx}
  \end{center}
\end{example}

\paragraph{Where do $\domain$ and $\labelling$ come from?} 
There is significant work in generating HL state and temporal abstractions in this form from raw sensory data (e.g., \citep{konidaris.etal.2014,konidaris.etal.2018,asai.fukunaga.2018,asai.etal.2022,chitnis.etal.2022,silver.etal.2021,james.etal.2022,silver.etal.2022,silver.etal.2023,shah.etal.2024,xi.etal.2024,huang.etal.2025,li.etal.2025}), whose progress has been propelled \citep{liu.etal.2023,han.etal.2024,athalye.etal.2026,ahmed.etal.2025,liang.etal.2025,li.etal.2025,yang.etal.2025,shen.etal.2026,liang.etal.2026} by vision language model advances \citep{radford.etal.2021,team.etal.2023}.
The result is an HL domain description $\domain$ along with a labelling function $\labelling$ that maps LL states to HL states.
Specifically, $\domain$ may be manually generated by a human and that $\domain$ and $\labelling$ can be generated via a VLA or learned.

\paragraph{Problem Statement Summary}
We can now define the problem:
given an HL domain $\domain$ that describes the abstract actions available to the agent, a labelling function $\labelling$ that maps LL states to HL states, and a finite dataset $\dataset = \datasetDefn$ consisting of LL state-action demonstrations along with HL goals, our objective is to learn a goal-conditioned policy $\policyDefn$ conditioned on LL states and HL goals that chooses low-level actions to achieve the goal.

% !TEX root =mainfile.tex

\section{Learning Bilevel Policies for Bilevel Planning}\label{sec:method}
We learn bilevel policies from the given HL-goal labelled LL demonstrations $\dataset$, using the HL domain $\domain$ and labelling function $\labelling$.
A \emph{bilevel policy} (BP) consists of a pair of policies $(\hlp, \llp)$ of the forms $\hlpDefn$ and $\llpDefn$, respectively.
Both policies are goal-conditioned and can be reused across different problems.
Together with the labelling function $\labelling$, the pair of policies can be composed into a single policy operating on the LL environment conditioned on HL goals:
\begin{align}
    \textstyle\pi(\lla \mid \lls, \hlg) = \sum_{\hla} \llp(\lla \mid \lls, \hla, \hlg) \cdot \hlp(\hla \mid \labelling(\lls), \hlg).
    \label{eq:bilevel}
\end{align}%
In the BP framework, the HL and LL policies are learned separately, with the main idea being that HL policies $\hlp$ capture HL decision making, while LL policies $\llp$ provide LL realisations of the HL actions prescribed by $\hlp$.
While $\llp$ is learned directly from $\dataset$ via imitation learning, the HL policy is learned by first constructing an abstract dataset $\hldataset$ by applying $\labelling$ to $\dataset$, and then inductively learning a symbolic policy $\hlp$ from $\hldataset$.
We realise the HL policy $\hlp$ as a set of first-order, condition-action rules (\Cref{ssec:hl-policy}) and the LL policy $\llp$ as a graph neural network (\Cref{ssec:ll-policy}).

\begin{figure}
\centering
    \input{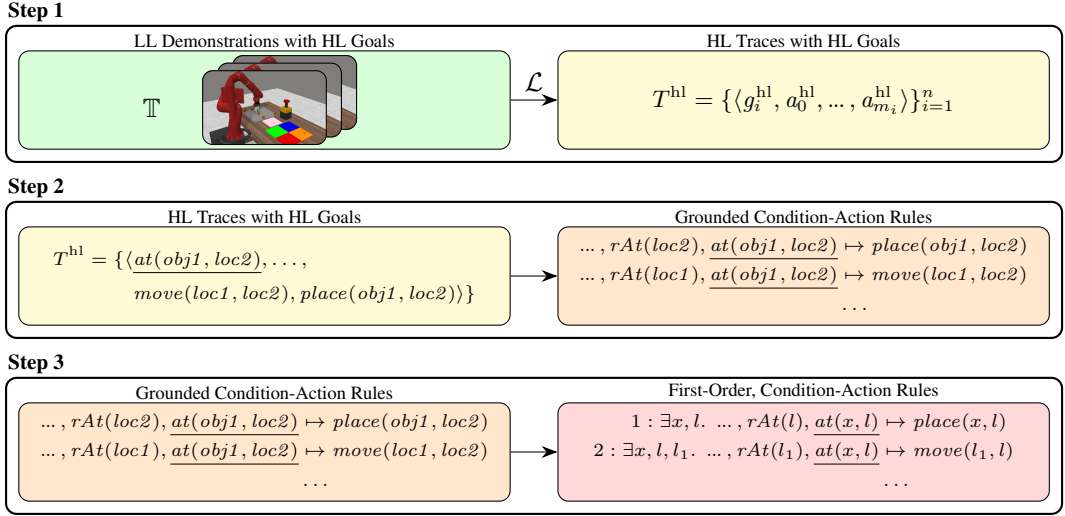}
    \caption{The HL policy $\hlpDefn$ learning process. 
    \textbf{Step 1}: we use the labelling function $\labelling$ to construct HL traces from the LL demonstrations paired with HL goals. 
    \textbf{Step 2}: we utilise goal regression to extract condition-action rules from the HL traces and goals (underlined).
    \textbf{Step 3}: we inductively generalise the rules by replacing objects with variables to produce symbolic policies.
    }
    \label{fig:hl-policy}
\end{figure}

\subsection{Learning HL Policies via Goal Regression and Inductive Generalisation}\label{ssec:hl-policy}
We construct HL policies $\hlpDefn$ from LL demonstrations paired with HL goals in 3 steps illustrated in \Cref{fig:hl-policy}:
(1) we first generate HL traces that are abstractions of our LL demonstrations;
(2) we then utilise goal regression to extract a set of condition-action rules from such HL traces; and
(3) we finally utilise inductive generalisation to produce interpretable symbolic policies in the form of first-order, condition-action rules. 
Steps 2 and 3 are achieved in milliseconds.
In what follows, we first define the form of our HL policies and a means of grounding them for execution with respect to a particular problem instance.
We then elaborate on the 3 steps required to generate HL policies.

\paragraph{HL policy representation}
Our HL policies $\hlp$ consist of sets of first-order, condition-action rules $r$ with associated priority values related to goal proximity.
These first-order, condition-action rules are tuples of the form $\langle\val(r),\allowbreak \var(r),\allowbreak \sCond(r),\allowbreak \gCond(r),\allowbreak \aHead(r)\rangle$ where
$\val(r)$ denotes the rule priority,
$\var(r)$ denotes a set of variables,
$\sCond(r)$ and $\gCond(r)$ are first-order conjunctive formulas over predicates in $\predicates$ representing the state and goal conditions, and 
$\aHead(r)$ is an action schema denoting the action to take if the conditions hold.
\Cref{ex:pick-place-policy} illustrates such an HL policy.

\begin{example}\label{ex:pick-place-policy}
    A first-order, condition-action policy for the pick and place domain in \Cref{ex:pick-place}, where goal conditions $\gCond$ are indicated by underlining:
    \small
    \scriptsize\begin{align*}
    \overbrace{\!1\!:\!}^{\val(r)} \overbrace{\exists x, l}^{\var(r)}.\! \; \overbrace{\mi{hold}(x) \ruleAnd \mi{rAt}(l)}^{\sCond(r)}
    \ruleAnd \overbrace{\underline{\mi{at}(x, l)}}^{\gCond(r)}
    &\!\mapsto\! \overbrace{\mi{place}(x, l)}^{\aHead(r)}
    & 3\!:\!\exists x, l, l_1.\! \; \mi{at}(x, l_1) \ruleAnd \mi{free}() \ruleAnd \mi{rAt}(l_1)
    \ruleAnd \underline{\mi{at}(x, l)}
    &\!\mapsto\! \mi{pick}(x, l_1) \\
    2\!:\!\exists x, l, l_1.\! \; \mi{hold}(x) \ruleAnd \mi{rAt}(l_1)
    \ruleAnd \underline{\mi{at}(x, l)}
    &\!\mapsto\! \mi{move}(l_1, l)
    &4\!:\!\exists x, l, l_1, l_2.\! \; \mi{at}(x, l_1) \ruleAnd \mi{free}() \ruleAnd \mi{rAt}(l_2)
    \ruleAnd \underline{\mi{at}(x, l)}
    &\!\mapsto\! \mi{move}(l_2, l_1)
\end{align*}%
\end{example}%

\newcommand*{\groundrules}[1][\hlp]{\ensuremath{\ground(#1)}\xspace}
In order to select the next action according to an HL policy \hlp, variables within the rules must be replaced with appropriate objects.
Given a rule $r$ of $\hlp$, a \emph{grounding} of $r$ is an assignment of the variables $\var(r)$ in $\sCond(r)$, $\gCond(r)$, and $\aHead(r)$ with objects from $\objects$.
A ground rule $r \in \groundrules$ is \emph{applicable} in an HL state $\hls$ with goal $\hlg$ if $\sCond(r) \subseteq \hls$ (i.e., the state condition is satisfied by the state) and $\gCond(r) \subseteq \hlg \setminus \hls$ (i.e., the goal condition is satisfied by the set of unachieved goal facts).
A policy $\hlp$ is executed by finding, via database query algorithms, a ground rule $r$ with lowest $\val(r)$ value and ties broken arbitrarily that is applicable in $\hls$ and $\hlg$.
Implicitly, this defines a probability distribution $\hlp(\hla \mid \hls, \hlg)$ over the grounded HL actions,
where $\hlp(\hla \mid \hls, \hlg) = 1$ if $\hla$ is selected by $\hlp$ in state $\hls$ given goal $\hlg$, and $0$ otherwise.

\paragraph{Step 1: construct HL traces}
From an LL trace $\gena{\hlg_i, \lls_0, \lla_0, \dotter, \lls_{m}, \lla_{m}} \in \dataset$, the domain $\domain$, and the labelling function $\labelling$, we can extract an HL trace $\gena{\hlg_i, \hla_0, \dotter, \hla_{n}}$ as follows: starting with the HL abstraction $\hls_0 = \labelling(\lls_0)$ of the initial LL state $\lls_0$, we iteratively determine the LL state where the HL abstraction changes, i.e., where $\labelling(\lls_{j}) \neq \labelling(\lls_{j-1})$, which defines the next HL state $\hls_i = \labelling(\lls_j)$.
We then find an applicable HL action $\hla_i$ which causes the change from $\hls_{i-1}$ to $\hls_{i}$, i.e., an $\hla_i$ such that $\pre(\hla_i) \subseteq \hls_{i-1}$ and for some effect $j$, we have $\hls_i = (\hls_{i-1} \cup \add_j(\hla_i)) \setminus \del_j(\hla_i)$.
By applying this extraction for each LL trace from $\dataset$, we obtain a HL dataset $\hldataset = \hldatasetDefn$.

\paragraph{Step 2: extract rules}
Our method converts HL action traces into ground condition-action rules via \emph{goal regression}~\citep{fikes.etal.1972,waldinger.1977,lozano.etal.1984,reiter.1991,reiter.2001,fritz.mcilraith.2007}.
Goal regression computes the minimal condition under which an action achieves a goal;
when applied recursively, it determines under which condition an action trace will lead to a goal state.
Formally, a set of facts $\hlg$ is regressable over an HL action $\hla$ if for every nondeterministic outcome $i$ of $\hla$ we have $\del_i(\hla) \cap \hlg = \emptyset$, in which case we define
\begin{align}
    \regr(\hlg, \hla) = \seta{(\hlg \setminus \add_i(\hla)) \cup \pre(\hla) \mid i = 1, \ldots, n},
\end{align}
and $\regr(\hlg, \hla) = \emptyset$ otherwise.
Notably, goal regression is also a powerful tool for enabling open-world and generalised planning \citep{gretton.thiebaux.2004,sanner.boutilier.2009,liu.etal.2025}.

\paragraph{Step 3: inductive generalisation}
Finally, we \emph{lift} the (ground) condition-action rules to obtain first-order rules by replacing ground objects with fresh variables.
The resulting rules are abstracted from specific problem instances and thus generalise to arbitrary numbers of objects.
Given an HL action $\hla$ and sets of facts $\hls$ and $\hlg$, let $o_1, \ldots, o_q$ be the union of all objects from the action and facts.
We then define the set of variables $\var = \set{v_1, \ldots, v_q}$ and mapping $g: \objects \to \var$ defined by $g(o_i) = v_i$.
Then we define the lifting operator as 
\begin{align}
    \lift(\hla, \hls, \hlg) = \gena{\var, \sCond, \gCond, \aHead}
    \label{eq:lifting}
\end{align}
where $\sCond$ and $\gCond$ are the conjunctions of the facts in $\hls$ and $\hlg$ respectively with $g$ applied to their objects, and $\aHead$ is the action schema of $\hla$ with $g$ applied on its objects.

    \ifwrapHLLearnAlgo
    \setlength{\intextsep}{0pt}
    \begin{wrapfigure}{r}{0.4\textwidth}
        \small
        \hfill
        \begin{minipage}{0.38\textwidth}
        \input{algos/learn_hl.tex}
        \end{minipage}
    \end{wrapfigure}
\else
    % !TEX root =mainfile.tex

\ifwrapHLLearnAlgo
\begin{algorithm}[H]
\else
\begin{algorithm}[ht]
\fi
\algorithmSize
  \DontPrintSemicolon
  \LinesNumbered
  \RestyleAlgo{ruled}
    \caption{HL policy learning}\label{alg:hl-policy-learning}
    \KwInput{LL demonstrations with HL goals $\dataset$.}
    \KwOutput{HL policy $\hlp$.}
    {
        \tcp{\color{blue}Step 1: Construct HL Traces}
        $\hldataset \gets \constructFrom(\dataset)$\; \label{line:construct-hl-traces}
        $\hlp \gets \emptyset$\; \label{line:hlp-init}
        \For{
            $\gena{\hlg, \hla_0, \ldots, \hla_{m}} \in \dataset^{\hl}$
        }{ \label{line:hlp-loop}
            $S \gets \seta{\hlg}$\; \label{line:prob-init}

            \For{$j = m, \ldots, 0$}{ \label{line:plan-loop}
                $S^{\textit{next}} \gets \emptyset$\; \label{line:next-init}
                \For{$\hls \in S$}{ \label{line:loop-hls}
                    \tcp{\color{blue}Step 2: Extract Rules}
                    $S' \gets \regr(\hls, \hla_j)$\; \label{line:regr}
                    \For{${\hls}' \in S'$}{
                        ${\hlg}' \gets \hlg \cap \succ({\hls}', \gena{{\hla_j}', \ldots, {\hla_{m}}'})$\\ \label{line:succ}
                        \tcp{\color{blue}Step 3: Inductive Generalisation}
                        $r \gets \lift(\hla_j, {\hls}', \hlg)$\; \label{line:lift}
                        $\hlp \gets \hlp \cup \seta{\gena{r, m-j}}$\; \label{line:hlp-update}
                    }
                    $S^{\textit{next}} \gets S^{\textit{next}} \cup S'$\; \label{line:next-update}
                }
                $S \gets S^{\textit{next}}$\; \label{line:next-assign}
            }
        }
        \Return $\hlp$\; \label{line:hlp-return}
    }
\end{algorithm}%
\fi

\paragraph{Putting the steps together}
\Cref{alg:hl-policy-learning} describes the procedure of learning an HL policy $\hlp$ from LL demonstrations $\dataset$ via goal regression.
It begins by initialising an empty policy (Line 2) and then iterates over the HL trajectories in $\hldataset$ (Line 3) constructed from the LL demonstrations (Line 1).
For each trajectory, it initialises $S$ to be the singleton set containing the goal to be regressed (Line 4).
Then it iterates over the HL actions in reverse order (Line 5), and at each step regresses the current goals in $S$ (Line 7) over the current HL action $\hla_j$ to obtain the next set of subgoals $S'$ (Lines 6, 8, 13 and 14).
The actions and regressed goals in each step $j$ are lifted into first-order rules $r$ (Line 11) alongside ${\hlg}'$, the subset of $\hlg$ that is achieved by the suffix of determinised actions from the trajectory (Line 10).
The rules are added to the policy $\hlp$ with priority $m-j$ (related to goal proximity) (Line 12) and the resulting policy $\hlp$ is returned in Line 15.

\paragraph{Generalisation to arbitrarily many objects}
The following theorem establishes the conditions under which \Cref{alg:hl-policy-learning} returns policies that can generalise to problems with arbitrarily many objects and hence arbitrarily long horizons.
The main conditions are that the problems (a) exhibit goals whose components can be achieved independently in any order, as is the case for a large proportion of real-world problems \citep{simon.1956,korf.1987}, and (b) the goal components can be achieved with bounded size HL policies. These conditions are formally defined in Appendix \ref{app:theory} alongside the proof of the theorem.
Under these conditions, the proof follows by representing an infinite set $S$ of arbitrarily large HL problems as a finite collection of subproblems equivalent up to object renaming.
Each subproblem class exhibits a finite policy and we can compose all such policies into a single, finite policy that can solve every problem in $S$.

\begin{theorem}\label{thm:formal}
    Let $\domain = \gena{\predicates, \schemata}$ be an HL domain, $\labelling$ a labelling function, and $C \in \N$.
    There exists a finite dataset $\dataset$ such that the HL policy learned from $\dataset$ via \Cref{alg:hl-policy-learning} solves any HL planning problem $\hlprob$ conforming to $\domain$ and satisfying $C$-bounded goal independence.
\end{theorem}

\subsection{Learning LL Policies via Graph Neural Networks and Imitation Learning}\label{ssec:ll-policy}
We next realise our LL policies $\llpDefn$ with a graph neural network (GNN) architecture. 

\paragraph{LL policy representation}
We describe how to construct a GNN model from $\lls$, $\hla$ and $\hlg$, and a domain $\domain = \gena{\predicates, \schemata}$ that generates an embedding for predicting the LL action $\lla$.
Given vectors $\h_1 \in \R^n, \h_2 \in \R^m$, we denote their concatenation as $\h_1 \concat \h_2 \in \R^{n+m}$.
Given $i, n \in N$, we denote the one-hot encoding $\onehot{i}{n} \in \R^n$ with all 0's except 1 in the $i$-th element.
We assume an enumeration of predicates and schemata, and let $M$ be the maximum arity of the schemata (i.e., $\max_{a \in \schemata} \abs{\var(a)}$).

\begin{figure}
    \centering
    % !TEX root =mainfile.tex

\newcommand{\descSize}{\small}
\newcommand{\sss}{0.6}
\newcommand{\xshift}{\sss*1.7cm}
\newcommand{\yshift}{\sss*1.2cm}
\newcommand{\graphDiff}{\sss*4cm}
\newcommand{\cc}{0.35*\xshift}
\newcommand{\nodeAngleA}{60}
\newcommand{\nodeAngleB}{322.5}
\newcommand{\nodeDist}{\sss*2cm}
\newcommand{\captionY}{\sss*3cm}

\newcommand{\yymid}{-0.475cm}
\newcommand{\xxx}{0.125cm}

\begin{tikzpicture}[
  font=\scriptsize,
  >=Stealth,
  message/.style={
    ->,
    dashed,
    gray!75
  },
  lime/.style={
    rectangle, 
    draw, 
    rounded corners,
    inner sep=0pt, 
    minimum width=0.9cm,
    minimum height=0.6cm,
  },
  input/.style={
    lime,
    draw=obsCol,
    fill=orange!10,
    minimum height=\sss*5cm,
  },
  output/.style={
    lime,
    draw=actCol,
    fill=blue!8,
    minimum height=\sss*5cm,
  }
]

    \node at (-0.5, \captionY) {\descSize (a) \textbf{Initial Embeddings}};
    \node at (1.675*\graphDiff, \captionY) {\descSize (b) \textbf{Message Passing}};
    \node at (3.25*\graphDiff, \captionY) {\descSize (c) \textbf{Readout}};

    \begin{scope}
    \node[lime] (a_0)  at (0, 0)    {$\h_{\textit{pick}}^{(0)}$};
    \node[lime] (o1_0) at (\nodeAngleA:\nodeDist) {$\h_{\textit{obj}}^{(0)}$};
    \node[lime] (o2_0) at (\nodeAngleB:\nodeDist) {$\h_{\textit{loc}}^{(0)}$};
    \node[lime] (g_0) at (\cc, -2.5*\yshift) {$\h_{\textit{global}}^{(0)}$};
    \draw[thick](a_0) -- (o1_0);
    \draw[thick, shorten <= -0.045cm, shorten >= -0.045cm](a_0) -- (o2_0);

    \newcommand{\inn}{-1.2}
    \coordinate (xref) at (\inn, 0);
    \node (a_n) at (xref |- a_0) {$\h^{\textit{pick}}$};
    \node (o1_n) at (xref |- o1_0) {$\h^{\textit{obj}}$};
    \node (o2_n) at (xref |- o2_0) {$\h^{\textit{loc}}$};
    \node (g_n) at (xref |- g_0) {$\h_{\textit{global}}$};

    \newcommand{\innn}{\inn - 2}
    \node[input] at (\innn, \yymid) {};
    \node (lls) at (\innn, \yymid+1cm) {$\lls$};
    \node (hla) at (\innn, \yymid) {$\hla$};
    \node (hlg) at (\innn, \yymid-1cm) {$\hlg$};

    \draw[message] (a_n) -- (a_0);
    \draw[message] (o1_n) -- (o1_0);
    \draw[message] (o2_n) -- (o2_0);
    \draw[message] (g_n) -- (g_0);
    
    \newcommand{\llsx}{0.07cm}
    \newcommand{\hlgx}{0.02cm}
    \draw[message] ([xshift=\xxx+\llsx]lls.east) -- (o1_n.west);
    \draw[message] ([xshift=\xxx+\llsx]lls.east) -- (o2_n.west);
    \draw[message] ([xshift=\xxx+\llsx]lls.east) -- (g_n.west);
    \draw[message] ([xshift=\xxx]hla.east) -- (a_n.west);
    \draw[message] ([xshift=\xxx+\hlgx]hlg.east) -- (o1_n.west);
    \draw[message] ([xshift=\xxx+\hlgx]hlg.east) -- (o2_n.west);
    \draw[message] ([xshift=\xxx+\hlgx]hlg.east) -- (g_n.west);

    \end{scope}

    \begin{scope}[xshift=0.5*\graphDiff]
    \node at (0.25, 0) {\bf$\ldots$};
    \end{scope}

    \begin{scope}[xshift=\graphDiff]
    \node[lime] (a_1)  at (0, 0)    {$\h_{\textit{pick}}^{(l)}$};
    \node[lime] (o1_1) at (\nodeAngleA:\nodeDist) {$\h_{\textit{obj}}^{(l)}$};
    \node[lime] (o2_1) at (\nodeAngleB:\nodeDist) {$\h_{\textit{loc}}^{(l)}$};
    \node[lime] (g_1) at (\cc, -2.5*\yshift) {$\h_{\textit{global}}^{(l)}$};
    \draw[thick](a_1) -- (o1_1);
    \draw[thick, shorten <= -0.045cm, shorten >= -0.045cm](a_1) -- (o2_1);
    \end{scope}

    \begin{scope}[xshift=2*\graphDiff]
    \node[lime] (a_2)  at (0,0)    {$\h_{\textit{pick}}^{(l+1)}$};
    \node[lime] (o1_2) at (\nodeAngleA:\nodeDist) {$\h_{\textit{obj}}^{(l+1)}$};
    \node[lime] (o2_2) at (\nodeAngleB:\nodeDist) {$\h_{\textit{loc}}^{(l+1)}$};
    \node[lime] (g_2) at (\cc, -2.5*\yshift) {$\h_{\textit{global}}^{(l+1)}$};
    \draw[thick](a_2) -- (o1_2);
    \draw[thick, shorten <= -0.045cm, shorten >= -0.045cm](a_2) -- (o2_2);
    \end{scope}

    \begin{scope}[xshift=2.5*\graphDiff]
    \node at (0.25, 0) {\bf$\ldots$};
    \end{scope}

    \begin{scope}[xshift=3*\graphDiff]
    \node[lime] (a_l)  at (0,0)    {$\h_{\textit{pick}}^{(L)}$};
    \node[lime] (o1_l) at (\nodeAngleA:\nodeDist) {$\h_{\textit{obj}}^{(L)}$};
    \node[lime] (o2_l) at (\nodeAngleB:\nodeDist) {$\h_{\textit{loc}}^{(L)}$};
    \node[lime] (g_l) at (\cc, -2.5*\yshift) {$\h_{\textit{global}}^{(L)}$};
    \draw[thick](a_l) -- (o1_l);
    \draw[thick, shorten <= -0.045cm, shorten >= -0.045cm](a_l) -- (o2_l);

    \newcommand{\outt}{3.75cm}
    \node[output] at (\sss*\outt, \yymid) {};
    \node (a) at (\sss*\outt, \yymid) {$\lla$};
    \draw[message] (a_l) -- ([xshift=-\xxx]a.west);
    \draw[message] (o1_l) -- ([xshift=-\xxx]a.west);
    \draw[message] (o2_l) -- ([xshift=-\xxx]a.west);
    \draw[message] (g_l)  to[out=0, in=-120] ([xshift=-\xxx]a.west);

    \end{scope}

    \draw[message] (a_1) -- (a_2);
    \draw[message] (g_1) -- (g_2);
    \draw[message] (o1_1) -- (o1_2);
    \draw[message] (o2_1) -- (o2_2);

    \draw[message] (a_1.east) -- (o1_2.west);
    \draw[message] (a_1.east) -- (o2_2.west);
    \draw[message] (o1_1.east) -- (a_2.west);
    \draw[message] (o2_1.east) -- (a_2.west);

    \draw[message] (o1_1.east) to  (g_2.north west);
    \draw[message] (a_1.east)  to[out=0, in=120] (g_2.north west);
    \draw[message] (o2_1.east) to  (g_2.north west);

    \draw[message] (g_1.north) to  (o1_1.south);
    \draw[message] (g_1.north) to  (a_1.south);
    \draw[message] (g_1.north) to  (o2_1.south);

    \draw[message] (g_2.north) to  (o1_2.south);
    \draw[message] (g_2.north) to  (a_2.south);
    \draw[message] (g_2.north) to  (o2_2.south);
\end{tikzpicture}
    \caption{
        The LL policy $\llpDefn$ represented by a GNN.
        In this example, the input action is $\hla = \hlFont{pick}(\textit{obj}, \textit{loc})$, and the resulting output is $\lla$.
        Solid lines represent graph edges, and dashed lines represent how information is passed.
        \textbf{Bold font} indicates Euclidean vectors.
    }\label{fig:gnn}
\end{figure}

Let $\hla = a(o_1, \ldots, o_{n})$, $\lls = \gena{\x_e, \set{\x_o}_{o \in \objects}}$ and $\hls = \labelling(\lls)$.
The GNN inputs are initially encoded as a set of embeddings $\h_{\globalNode}, \h_{a}, \h_{o_1}, \ldots, \h_{o_n}$ representing a global node, an action node, and object nodes, respectively, where
$\h_{\globalNode} = \x_e \concat \sum_{p() \in \hls} \onehot{p}{\abs{\predicates}} \concat \sum_{p() \in \hlg} \onehot{p}{\abs{\predicates}}$ encodes the agent's state and nullary facts that are true in the HL state and goal,
$\h_{a} = \onehot{a}{\abs{\schemata}}$ encodes the action schema, and
$\h_{o_i} = \x_{o_i} \concat \sum_{p(o_i) \in \hls} \onehot{p}{\abs{\predicates}} \concat \sum_{p(o_i) \in \hlg} \onehot{p}{\abs{\predicates}} \concat \onehot{i}{M}$ encodes an object $o_i$'s state, unary facts instantiated with $o_i$ in the HL state and goal, and its position in the action $\hla$.

The GNN algorithm then consists of the following three operations, as illustrated in \Cref{fig:gnn}.
The $\W$ symbols denote learnable weights.
\begin{enumerate}[label=(\alph*)]
    \item The initial encodings are embedded: $\h_{\globalNode}^{(0)} = \W_g^{(0)} \h_{\globalNode}$, $\h_a^{(0)} = \W_a^{(0)} \h_a$, $\h_{o_i}^{(0)} = \W_o^{(0)} \h_{o_i}$.
    \item Next, $L$ rounds of {message passing} are performed.
Object embeddings are aggregated via element-wise max
$\h_o^{(l)} = \max_{i=1, \ldots, n} \h_{o_i}^{(l)}$ 
and node embeddings are updated by
$\h_{\globalNode}^{(l+1)} = \sigma(\W_g^{(l)}(\h_{\globalNode}^{(l)} + \h_a^{(l)} + \h_o^{(l)}))$, 
$\h_a^{(l+1)} = \sigma(\W_a^{(l)}(\h_{\globalNode}^{(l+1)} + \h_a^{(l)} + \h_o^{(l)}))$, and 
$\h_{o_i}^{(l+1)} = \sigma(\W_o^{(l)}(\h_{\globalNode}^{(l+1)} + \h_a^{(l)} + \h_{o_i}^{(l)}))$ 
for nonlinearity $\sigma$.
    \item Finally, a feed-forward {readout} layer is then applied to the final embeddings $\h_{\globalNode}^{(L)} + \h_a^{(L)} + \h_{o}^{(L)}$ to predict the LL action $\lla$.
\end{enumerate}

\paragraph{Learning method}
Following the HL learning method,  we can compute a single dataset consisting of tuples $\gena{\lls, \hla, \hlg, \lla}$ from the LL demonstrations corresponding to inputs and outputs of $\llp$.
We optimise the GNN to minimise the mean squared error (MSE) between the predicted LL actions and the ground truth LL actions in the dataset $\dataset$.
In place of MSE optimisation, we could alternatively learn a generative diffusion policy \citep{ho.etal.2020,chi.etal.2023}. We leave exploration of this alternative to future work.

% !TEX root =mainfile.tex

\section{Experiments}\label{sec:experiments}
We implement our approach from \Cref{sec:method} in a system called \genplan{} and conduct experiments on extended MetaWorld~\citep{yu.etal.2019,mclean.etal.2025} benchmarks.
We compare our approach to various VLA, end-to-end, and planning systems to demonstrate the efficacy of \genplan{} for learning bilevel policies that can generalise to longer horizons and more objects than seen in its training demonstrations.

\paragraph{Setup}
We evaluate 7 imitation learning architectures and 2 VLA baselines across 8 MuJoCo robot simulation environments.
For each approach and environment, we train 3 models on 200 goal-achieving trajectories from problems with 3 objects.
We train the GNN LL policy in \Cref{ssec:ll-policy} with hyperparameters specified in Appendix \ref{app:hyperparams}.
Each trained policy is evaluated on problems with 1 to 10 objects from each environment with 10 episodes for each object and at most 2048$n$ steps for each episode where $n$ is the number of objects.
As a result, we conduct up to $9 \times 8 \times 3 \times 10 \times 10 = 21,600$ episodes' worth of evaluations.
Training is performed with an NVIDIA L40 GPU (48GB memory) while evaluation is performed on Intel Xeon Gold 6338 CPUs (2.00 GHz).

\newcommand{\www}{0.3\textwidth}
\begin{wraptable}{r}{\www}
    \centering
    \renewcommand{\arraystretch}{0.9}
    \caption{Environments and their uncertainty attributes.}\label{tab:setup}
    \scriptsize
    % !TEX root =mainfile.tex

\newcommand{\yy}{\cmark}
\newcommand{\nn}{}
    \tabcolsep 2pt
    \begin{tabularx}{\www}{l *{8}{>{\centering\arraybackslash}X}}
        % \toprule
        % & \multicolumn{8}{c}{Environments}
        % \\
        % \cmidrule(l){2-9}
        Uncertainty
        % ---------------------------
        & \header{\blocks} & \header{\blocksNoisy} & \header{\colour} & \header{\colourNoisy} & \header{\factory} & \header{\factoryNoisy} & \header{\gacha} & \header{\gachaNoisy} \\
        \cmidrule{1-1} \cmidrule(l){2-9}
        Exogenous
        & \nn & \yy & \nn & \yy & \nn & \yy & \nn & \yy \\
        Endogenous
        & \nn & \nn & \yy & \yy & \nn & \nn & \yy & \yy \\
        State
        & \nn & \nn & \nn & \nn & \yy & \yy & \yy & \yy \\
        % \bottomrule
    \end{tabularx}

\end{wraptable}

\paragraph{Environments}
We evaluate the aforementioned methods on 8 environments that compositionally extend the tasks from the MetaWorld benchmark suite~\citep{yu.etal.2019,mclean.etal.2025}.
There are 4 stationary environments (\blocks{}, \factory{}, \colour{}, \gacha{}) each with a corresponding variant (\blocksNoisy{}, \factoryNoisy{}, \colourNoisy{}, \gachaNoisy{}) where objects may randomly teleport around throughout the episode.
Additional environment semantics and visualisations are provided in Appendix \ref{app:benchmarks}.
Each environment induces a problem distribution with different initial states and goals, parameterised by the number of objects.
\Cref{tab:setup} summarises the environments and their uncertainty attributes:
\textbf{exogenous} transition uncertainty that is not modelled in the HL abstraction,
\textbf{endogenous} transition uncertainty that is modelled in the HL abstraction, and
\textbf{state} uncertainty and partial observability of object attributes and locations.

\paragraph{Methods}
We evaluate and compare against the following methods to demonstrate the robustness of our approach to various forms of uncertainty. 
We introduce a mix of baselines that correspond to end-to-end, planning, and replanning approaches, with more details in Appendix \ref{app:methods}.
\begin{itemize}
    \item \textbf{\oracle{}}: a hand-coded policy that serves as the expert and upper bound for the experiments
    \item \textbf{\vla{}}: the SmolVLA open-source 0.45B parameter VLA model \citep{shukor.etal.2025} which was shown to outperform state-of-the-art VLAs with billions of parameters.
    \item \textbf{\vlamw{}}: \vla{} fine-tuned on 2,500 episodes of the original MetaWorld benchmark suite
    \item \textbf{\purenn{}}: a GNN policy operating over LL observations with no access to HL facts and actions
    \item \textbf{\pddlnn{}}: a GNN policy operating over LL observations and HL facts but no access to HL actions
    \item \textbf{\detplan{}}: the GNN policy from \Cref{ssec:ll-policy} with a planner \citep{helmert.2006} in place of an HL policy
    \item \textbf{\regplan{}}: the GNN policy from \Cref{ssec:ll-policy} with a nondeterministic planner \citep{muise.etal.2024} for generating a problem-specific HL policy
    \item \textbf{\detreplan{}}: same as \detplan{} but replan (recompute a new plan) if no HL action is returned due to uncertainty \citep{pettersson.2005,yoon.etal.2007,little.etal.2007,fritz.mcilraith.2007,curtis.etal.2022a,curtis.etal.2024}
    \item \textbf{\regreplan{}}: same as \regplan{} but replan (compute a new policy) if no HL action is returned
    \item \textbf{\genplan{}} {\color{blue}(new)}: the bilevel policy for bilevel planning approach using the problem-agnostic HL policy from \Cref{ssec:hl-policy} and the GNN LL policy from \Cref{ssec:ll-policy} 
\end{itemize}

\paragraph{(Q1) How do VLAs perform on the benchmarks?}
From \Cref{tab:results} we observe that pretrained VLA models perform poorly on the benchmarks and fail to solve any task.
We observed that this is consistent with existing evaluations of VLAs for the original MetaWorld benchmarks \citep{shukor.etal.2025}, namely how SmolVLA and other VLA approaches achieve at most 45\% success rate on \textit{hard} MetaWorld problems \citep{seo.etal.2022}, which includes the $n=1$ case of Blocks.
Furthermore, we observed that \vlamw{} can only solve Blocks for $n=1$ when the block and goal locations are in the same area as in the training set, but not outside.
However, we note that VLAs are operating on image and text inputs of the environment, while all other methods assumed processed representations of the inputs.

\paragraph{(Q2) Does \genplan{} generalise to longer horizons than seen in training?}
From \Cref{fig:results}, we observe that \genplan{} overall is able to generalise to environments beyond the number of objects that it is trained on, and often generalises better than the end-to-end GNN baselines \purenn{} and \pddlnn{} whose performance often drops on environments beyond 6 objects.
This can be attributed to the expressivity limits of GNNs \citep{morris.etal.2019,xu.etal.2019a} that prohibit them from extrapolating to longer-horizon HL planning problems as empirically and theoretically shown in previous works \citep{staahlberg.etal.2022,silver.etal.2022,silver.etal.2023}.
The best performing seeds often match the upper bound performance of the oracle.
From \Cref{tab:results}, \genplan{} matches or outperforms all other baselines with two exceptions.
Failures often occur in the covariate shift in the LL execution of HL actions, which suggests that further improvements can be gained from additional post-training, e.g., via RL or DAgger \citep{ross.etal.2011}.

\begin{table}[t]
    \scriptsize
    \tabcolsep 3pt
    \caption{
        Average success rate ($\uparrow$) and standard deviation of various methods across number of objects and environments.
        The top 2 scores per column are highlighted.
    }\label{tab:results}
    \begin{tabularx}{\textwidth}{ccYYYYYYYYY}
\toprule
Type & Method & {\blocks{}} & {\blocksnoisy{}} & {\factory{}} & {\factorynoisy{}} & {\colour{}} & {\colournoisy{}} & {\gacha{}} & {\gachanoisy{}} & Total \\
\midrule
\multirow{2}{*}{VLA} & \vla{} & \tableCell{0.00}{0.0} & \tableCell{0.00}{0.0} & \tableCell{0.00}{0.0} & \tableCell{0.00}{0.0} & \tableCell{0.00}{0.0} & \tableCell{0.00}{0.0} & \third{\tableCell{0.00}{0.0}} & \third{\tableCell{0.00}{0.0}} & \tableCell{0.00}{0.0} \\
 & \vlamw{} & \tableCell{0.00}{0.0} & \tableCell{0.00}{0.0} & \tableCell{0.00}{0.0} & \tableCell{0.00}{0.0} & \tableCell{0.00}{0.0} & \tableCell{0.00}{0.0} & \third{\tableCell{0.00}{0.0}} & \third{\tableCell{0.00}{0.0}} & \tableCell{0.00}{0.0} \\
\cmidrule{2-11}
\multirow{2}{*}{GNN} & PureNN & \tableCell{0.00}{0.0} & \tableCell{0.00}{0.0} & \tableCell{0.00}{0.0} & \tableCell{0.00}{0.0} & \tableCell{0.00}{0.0} & \tableCell{0.00}{0.0} & \third{\tableCell{0.00}{0.0}} & \third{\tableCell{0.00}{0.0}} & \tableCell{0.00}{0.0} \\
 & PddlNN & \tableCell{0.33}{0.4} & \tableCell{0.20}{0.3} & \third{\tableCell{0.18}{0.3}} & \tableCell{0.07}{0.2} & \tableCell{0.12}{0.2} & \second{\tableCell{0.36}{0.3}} & \second{\tableCell{0.51}{0.3}} & \second{\tableCell{0.60}{0.4}} & \tableCell{0.30}{0.2} \\
\cmidrule{2-11}
\multirow{4}{*}{Planning} & DetPlan & \third{\tableCell{0.90}{0.1}} & \tableCell{0.34}{0.3} & \tableCell{0.00}{0.0} & \tableCell{0.00}{0.0} & \tableCell{0.00}{0.0} & \tableCell{0.00}{0.0} & \third{\tableCell{0.00}{0.0}} & \third{\tableCell{0.00}{0.0}} & \tableCell{0.15}{0.3} \\
 & NdtPlan & \first{\tableCell{\mathbf{0.99}}{0.0}} & \third{\tableCell{0.39}{0.4}} & \tableCell{0.00}{0.0} & \tableCell{0.00}{0.0} & \first{\tableCell{\mathbf{0.61}}{0.2}} & \tableCell{0.01}{0.0} & \third{\tableCell{0.00}{0.0}} & \third{\tableCell{0.00}{0.0}} & \tableCell{0.25}{0.4} \\
 & DetReplan & \second{\tableCell{0.98}{0.0}} & \first{\tableCell{\mathbf{0.97}}{0.1}} & \second{\tableCell{0.94}{0.1}} & \second{\tableCell{0.59}{0.3}} & \tableCell{0.00}{0.0} & \tableCell{0.00}{0.0} & \third{\tableCell{0.00}{0.0}} & \third{\tableCell{0.00}{0.0}} & \third{\tableCell{0.44}{0.4}} \\
 & NdtReplan & \first{\tableCell{\mathbf{0.99}}{0.0}} & \first{\tableCell{\mathbf{0.97}}{0.0}} & \second{\tableCell{0.94}{0.1}} & \third{\tableCell{0.51}{0.3}} & \third{\tableCell{0.26}{0.3}} & \third{\tableCell{0.17}{0.3}} & \third{\tableCell{0.00}{0.0}} & \third{\tableCell{0.00}{0.0}} & \second{\tableCell{0.48}{0.4}} \\
\cmidrule{2-11}
{\color{blue}(new)} & \genplan{} & \first{\tableCell{\mathbf{0.99}}{0.0}} & \second{\tableCell{0.95}{0.1}} & \first{\tableCell{\mathbf{0.97}}{0.0}} & \first{\tableCell{\mathbf{0.62}}{0.3}} & \second{\tableCell{0.54}{0.2}} & \first{\tableCell{\mathbf{0.99}}{0.0}} & \first{\tableCell{\mathbf{0.65}}{0.3}} & \first{\tableCell{\mathbf{0.64}}{0.3}} & \first{\tableCell{\mathbf{0.79}}{0.2}} \\
\bottomrule
\end{tabularx}

\end{table}

\begin{figure}[t]
    \includegraphics{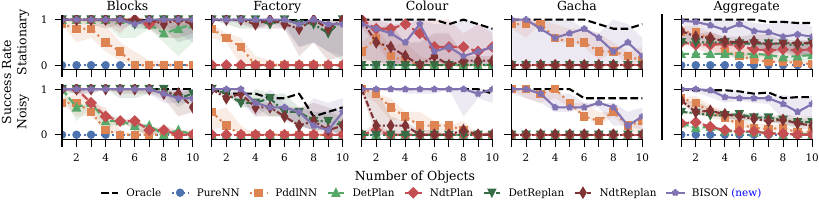}
    \caption{
        Median (line) and range (shaded) of success rate ($\uparrow$) of methods across number of objects and environments.
        Results for VLA baselines are omitted as they do not complete any tasks.
    }\label{fig:results}
\end{figure}

\paragraph{(Q3) Is \genplan{} robust to uncertainty and open-world environments?}
As we have seen in the previous question and in \Cref{tab:results}, \genplan{} generally performs best overall, with the exception of two environments where variance in LL policy execution causes failures.
Next, we observe that the next best performing approach is \ndtreplan{} which computes a new HL policy $\hlp$ per environment episode to be used in the bilevel policy execution framework, and recomputes $\hlp$ if it fails to return an action.
However, we observe that even with replanning, \ndtreplan{} is less robust to the more complex Colour environments, and furthermore cannot support the open-world Gacha environment.

\begin{wrapfigure}{r}{0.3\textwidth}
    \includegraphics{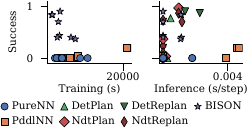}
    \caption{Success rate ($\uparrow$) vs. time ($\downarrow$) for training and inference with 10 objects.}
    \label{fig:time}
\end{wrapfigure}

\paragraph{(Q4) Is \genplan{} more efficient than (re)planning and end-to-end approaches?}
From \Cref{fig:time}, we observe that the end-to-end \purenn{} and \pddlnn{} approaches take longer to train and execute as they encode more information in the LL GNN policy architecture than the one from \Cref{ssec:ll-policy}:
\purenn{} encodes all LL information and \pddlnn{} encodes all LL and HL information, while \genplan{}'s LL policy only encodes the HL action and corresponding object information.
Single-shot planners (\detplan, \ndtplan) are comparable to \genplan{} as planning costs are amortised over the episode, but achieve poor performance.
Replanning methods (\detreplan, \ndtreplan) take at least as long as \genplan{} due to replanning upon encountering unexpected uncertainty.

\begin{wrapfigure}{r}{0.3\textwidth}
    \includegraphics{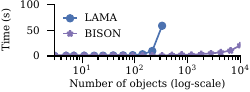}
    \caption{HL solving time ($\downarrow$) vs. number of objects.}
    \label{fig:planning}
\end{wrapfigure}

\paragraph{(Q5) Do learned HL policies generalise to arbitrary numbers of objects?}
We answer this question by inspecting the learned policies and evaluating their performance on HL planning problems over numbers of objects.
Learned HL policies in \genplan{} are symbolic and hence can be interpreted manually or by an LLM to generalise over arbitrary numbers of objects, as displayed in Appendix \ref{app:interpret}.
We also perform experiments on corresponding PDDL \blocks{} planning problems ranging from 3 to 10,000 blocks.
We compare \genplan{}'s solving time to a state-of-the-art PDDL planner, LAMA \citep{richter.westphal.2010}, with a 100s timeout.
We observe from \Cref{fig:planning} that \genplan{} solves problems with 10,000 blocks in a few seconds while LAMA struggles with a few hundred objects.

\section{Discussion, Limitations, and Conclusion}\label{sec:conclusion}
We introduced \genplan{}, a system to learn bilevel policies $(\hlp, \llp)$ from LL demonstrations paired with HL goals that generalise to long-horizon problems.
We learn $\hlp$ policies from abstracted symbolic traces of the LL demonstrations using goal regression and inductive generalisation to produce first-order, condition-action rules that can generalise to arbitrary numbers of objects, and 
realise our $\llp$ policies via graph neural networks with fewer than 33k parameters for executing HL actions returned by $\hlp$ in the LL action space.
The novelty of our approach is at least two-fold: (1) in contrast to existing bilevel planning approaches that search for HL plans over symbolic abstractions before refining them in the LL environment, \genplan{} inductively generalises HL policies from LL demonstrations without the need for search, and (2) the encoding of our policies as first-order, condition-action rules is novel in the context of bilevel planning and facilitates open-world planning and generalisation. Through this novel insight and approach, we can learn compact policies that demonstrate strong generalisation capabilities and efficiency compared to state-of-the-art VLA, end-to-end, and planning baselines on extended MetaWorld benchmarks.

\paragraph{Limitations}
Our approach inherits the limitations of bilevel planning, in particular, the assumption of a learned or given HL abstraction that provides useful guidance for LL planning (there is much research on how to find good abstractions as discussed in \Cref{sec:background}).
Another limitation of our approach is that we have no guarantees of optimality and generalisation of the LL policies.

\section*{Acknowledgements}
We gratefully acknowledge funding from the Natural Sciences and Engineering Research Council of Canada (NSERC) and the Canada CIFAR AI Chairs Program. 
Till Hofmann was funded by the Federal Ministry of Education and Research (BMBF) and the Ministry of Culture and Science of the German State of North Rhine-Westphalia (MKW) under the Excellence Strategy of the Federal Government and the L\"{a}nder. 
Resources used in preparing this research were provided, in part, by the Province of Ontario, the Government of Canada through CIFAR, and companies sponsoring the Vector Institute.
We thank Frieda Rong and William Shen for helpful discussions.

%%%%%%%%%%%%%%%%%%%%%%%%%%%%%%%%%%%%%%%%%%%%%%%%%%%%%%%%%%%%%%%%%%%%%%%%%%%%%%
\bibliography{dillon}
\bibliographystyle{plain}

%%%%%%%%%%%%%%%%%%%%%%%%%%%%%%%%%%%%%%%%%%%%%%%%%%%%%%%%%%%%%%%%%%%%%%%%%%%%%%
\clearpage
\appendix

\crefalias{section}{appendix}
\crefalias{subsection}{appendix}

\begin{center}
    \LARGE {\bf Appendix}
\end{center}

\tableofcontents
\clearpage

% !TEX root =mainfile.tex

\section{Summary of Definitions and Notation}\label{app:notation}
We use $\llFont{B}$old $\llFont{F}$ace font to distinguish low-level (LL) components from high-level (HL) components denoted via $\mathcal{C}$alligraphic and $\mathit{I}$talics font.
\Cref{tab:notation} summarises the definitions and notation used in the paper.

\begin{table}[ht!]
    \caption{Summary of definitions and notation.}\label{tab:notation}
    \footnotesize
    \begin{tabularx}{\textwidth}{l X p{4.6cm}}
        \toprule
        Term & Description & Notation \\
        \midrule
        problem & a tuple of states, actions, initial state and goal states & $\problem = \langle \states, \actions, \init, \goal \rangle$ \\
        bilevel problem & a LL and HL problem and a labelling function & $\bilevelProb=\bilevelDefn$ \\
        policy & a probability distribution over actions conditioned on a state and other arguments & $\pi(\action \mid \state, \ast)$ \\
        \midrule
        LL state & low-level object-centric, ego-centric state representation & $\lls = (\x_e, \set{\x_o}_{o \in \objects})$,\newline$\x_e \in \R^n, \x_o \in \R^m$ \\
        LL action & low-level action & $\lla \in \R^d$ \\
        LL policy & low-level policy & $\llpDefn$ \\
        \midrule
        HL state & high-level state consisting of a set of facts & $\hls$ \\
        HL action & high-level action induced by an action schema & $\hla = a(o_1, \ldots, o_{n_a})$ \\
        HL policy & high-level policy & $\hlpDefn$ \\
        \midrule
        domain & a set of predicates and action schemata & $\domain = \langle \predicates, \schemata \rangle$ \\
        labelling function & a function mapping LL states to HL states & $\labelling$ \\
        predicate & symbol with argument terms & $p(x_1, \ldots, x_{n_p})$ \\
        fact & predicate with all arguments instantiated with objects & $p(o_1, \ldots, o_{n_p})$ \\
        action schema & temporally abstracted action with argument terms & $a(x_1, \ldots, x_{n_a})$ \\
        \midrule
        dataset & LL demonstrations paired with HL goals & $\dataset = \datasetDefn$ \\[1ex]
        HL dataset & HL action traces paired with HL goals & $\hldataset = \hldatasetDefn$ \\
        \midrule
        succession & the successors of a state under an action & $\succ(\state, \action)$ \\
        regression & the regressors of a goal under an action & $\regr(\goal, \action)$ \\
        \bottomrule
    \end{tabularx}
\end{table}

% !TEX root =mainfile.tex

\section{Theory}\label{app:theory}
We provide additional context and the proof of \Cref{thm:formal}.
The argument requires two ingredients:
(1) the HL component of $\bilevelProb$ admits a solution that proceeds by sequentially achieving singleton goal facts (Appendix \ref{app:gi}), and 
(2) $\dataset$ exhibits enough HL abstracted demonstrations $\hldataset$ to extract HL rules covering every relevant singleton-goal subproblem (Appendix \ref{app:equiv}).

\subsection{Goal Independence}\label{app:gi}
We assume HL problems are represented in the relational form from \Cref{sec:background}.
Given an HL problem $\hlprob = \gena{\states^{\hl}, \actions^{\hl}, \hls_0, \hlg}$, an HL state $\hls \in \states^{\hl}$, and a set of facts ${\hlg}'$, let $\hlprob_{\hls, {\hlg}'}$ denote the same HL problem with initial state replaced by $\hls$ and goal replaced by ${\hlg}'$, i.e. $\hlprob_{\hls, {\hlg}'} = \gena{\states^{\hl}, \actions^{\hl}, \hls, {\hlg}'}$.
Let $init(\hlprob)$ and $goal(\hlprob)$ denote the initial state and goal of $\hlprob$, respectively.

\begin{definition}[Goal Independence]\label{def:gi}
    An HL problem $\hlprob$ exhibits \emph{goal independence (GI)} if for any ordering $\hlg_1, \ldots, \hlg_n$ of the facts in $goal(\hlprob)$, the following procedure constructs a solution for $\hlprob$:
    set $\hlp \la \emptyset$ and $S \gets \seta{init(\hlprob)}$ and for each $i = 1, \ldots, n$,
    (a) extend $\hlp$ with solution policies for $\hlprob_{\hls, \seta{\hlg_i}}$ for all $\hls \in S$ which do not delete any fact $\hlg_j$ for $i \not= j$, and
    (b) add all states in the support of an $\hlp_i$-trajectory ending in a state containing $\hlg_i$ to $S$ that are not yet covered by $\hlp$.
    We further say that $\hlprob$ exhibits \emph{$C$-bounded goal independence ($\text{GI}_C$)} for $C \in \N$ if every $\hlp_i$ in step (a) has a preimage of size at most $C$.
\end{definition}

A bilevel planning problem $\bilevelProb = \bilevelDefn$ exhibits (resp.~$C$-bounded) goal independence whenever its HL component $\hlprob$ does.
The bound $C$ in $\text{GI}_C$ ensures each singleton-goal subproblem admits bounded solutions which in turn bounds the size of the rule database extracted by goal regression in \Cref{alg:hl-policy-learning}.

\subsection{Object-Renaming Equivalence of HL Problems}\label{app:equiv}
Goal regression on a single demonstration produces lifted rules that, by virtue of being parameterised by free variables, generalise across object-renaming.
We make this precise via an equivalence relation on HL problems.

\begin{definition}[Object-Renaming Equivalence]\label{def:equiv}
    Two HL problems $\hlprob_1, \hlprob_2$ over the same domain $\domain$ are \emph{equivalent}, written $\hlprob_1 \sim \hlprob_2$, if there exists a bijection $f : \objects_1 \to \objects_2$ such that the lifted map
    \begin{align}
        F(\hls) := \seta{p(f(o_1), \ldots, f(o_n)) \mid p(o_1, \ldots, o_n) \in \hls}
    \end{align}
    satisfies $F(init(\hlprob_1)) = init(\hlprob_2)$ and $F(goal(\hlprob_1)) = goal(\hlprob_2)$.
\end{definition}

\begin{proposition}\label{prop:equiv}
    The relation $\sim$ is an equivalence relation.
    Moreover, suppose $\hlprob_1 \sim \hlprob_2$ via a bijection $f$, and extend $F$ to ground actions by $F(a(o_1, \ldots, o_n)) := a(f(o_1), \ldots, f(o_n))$.
    Given any HL policy $\hlp_1$ on $\hlprob_1$, define $F_* \hlp_1$ on $\hlprob_2$ by
    \begin{align}
        (F_* \hlp_1)(F(\hla) \mid F(\hls), F(\hlg)) := \hlp_1(\hla \mid \hls, \hlg).
    \end{align}
    Then $\sigma = \gena{\hls_0, \hls_1, \ldots}$ is a $\hlp_1$-trajectory in $\hlprob_1$ if and only if $F(\sigma) := \gena{F(\hls_0), F(\hls_1), \ldots}$ is an $(F_* \hlp_1)$-trajectory in $\hlprob_2$, and consequently $\hlp_1$ is a solution for $\hlprob_1$ if and only if $F_* \hlp_1$ is a solution for $\hlprob_2$.
    In particular, a sequence of HL actions $a_0(o_0^1, \ldots, o_0^{n_0}), \ldots, a_{m-1}(o_{m-1}^1, \ldots, o_{m-1}^{n_{m-1}})$ reaches a goal state in $\hlprob_1$ if and only if its image under $F$ reaches a goal state in $\hlprob_2$.
\end{proposition}
\begin{proof}
Reflexivity, symmetry, and transitivity follow from the analogous properties of bijective functions.
For the policy statement, action applicability ($\pre(\hla) \subseteq \hls$) and the successor function are determined syntactically by the names of predicates and schemata, which $f$ leaves unchanged; hence $F$ commutes with $\succ$ and $\hla$ is applicable in $\hls$ iff $F(\hla)$ is applicable in $F(\hls)$.
By construction $F_* \hlp_1$ assigns identical support, so $\hlp_1$- and $(F_* \hlp_1)$-trajectories are in bijection via $F$, and goal states correspond by $F(goal(\hlprob_1)) = goal(\hlprob_2)$.
The action-sequence claim is the special case where $\hlp_1$ is the deterministic open-loop policy that prescribes $a_i(o_i^1, \ldots, o_i^{n_i})$ at step $i$.
A consequence of \Cref{prop:equiv} is that lifted rules extracted from one HL problem transfer verbatim to all problems in the same equivalence class.
\end{proof}

\subsection{Generalisation Result}
We can now state the formal version of the informal theorem in \Cref{sec:method}.
The proof relies on a finite covering argument: under $\text{GI}_C$, the relevant singleton-goal subproblems form a finite collection up to $\sim$, so a finite dataset suffices for the HL rule database to generalise to every test problem.

\setcounter{theorem}{0}
\begin{theorem}
    Let $\domain = \gena{\predicates, \schemata}$ be an HL domain, $\labelling$ a labelling function, and $C \in \N$.
    There exists a finite dataset $\dataset$ such that the HL policy learned from $\dataset$ via \Cref{alg:hl-policy-learning} solves any HL planning problem $\hlprob$ conforming to $\domain$ and satisfying $C$-bounded goal independence ($\text{GI}_C$).
\end{theorem}

\begin{proof}
    Let $N = \max_{a \in \schemata} \abs{\var(a)}$ and $M = \max_{p \in \predicates} \arity(p)$ denote the maximum schema and predicate arities, respectively.

    Up to $\sim$, the set of singleton-goal HL problems on $\domain$ is finite.
    A singleton goal is a single fact $p(o_1, \ldots, o_{n_p})$ which, modulo bijective renaming of objects, is determined by a choice of $p \in \predicates$ together with an arrangement of its at most $M$ argument positions among at most $M$ distinct objects.
    There are therefore at most $\abs{\predicates} \cdot M^M$ singleton goals up to $\sim$.
    By the $\text{GI}_C$ assumption, every singleton-goal subproblem admits a solution policy with size $C$ and hence cycle-less trajectories with size less than $C$.
    A length-$k$ HL action sequence together with its singleton goal mentions at most $kN + M$ distinct objects, since each of the $k$ actions has at most $N$ parameter positions and the goal contributes at most $M$ further positions.
    By \Cref{prop:equiv} it suffices to count action sequences up to $\sim$:
    each of the $k$ action positions contributes at most $\abs{\schemata}$ schema choices and at most $\lr{kN + M}^N$ parameter-tuple choices.
    Hence the number of inequivalent rules extracted by regressing singleton-goal subproblems through length-at-most-$C$ action sequences is bounded by
    \begin{align}
        n \;\le\; \abs{\predicates} \cdot M^M \cdot \sum_{k=0}^{C} \biglr{\abs{\schemata} \cdot \lr{kN + M}^N}^k.
        \label{eq:finite-cover}
    \end{align}
    Choose $\dataset$ to contain, for each equivalence class of singleton-goal subproblem encountered as a step (a) instance in \Cref{def:gi}, at least one demonstration whose HL projection realises a solution policy of length at most $C$ for a representative of that class.
    Such a $\dataset$ is finite by \Cref{eq:finite-cover}.

    Now fix any HL problem $\hlprob$ satisfying the theorem assumptions and let $\hlg_1, \ldots, \hlg_n$ be any goal ordering.
    Let $\hlp$ denote the HL policy synthesised from $\dataset$ by \Cref{alg:hl-policy-learning}.
    By construction, $\hlp$ is the union of lifted rules obtained by goal regression of the demonstrations in $\dataset$.
    By Step~1 and \Cref{prop:equiv}, for every HL state $\hls$ reached during execution and every singleton subgoal $\hlg_i$ not yet achieved in $\hls$,
    there exists a grounded rule $r \in \ground(\hlp)$ that is applicable in state $\hls$ with goal $\hlg_i$.
    Executing $\aHead(r)$ thus prescribes a valid HL action that progresses $\hlprob_{\hls, \seta{\hlg_i}}$ and reduces the value of the next applicable rule in the HL policy.
    Iterating the construction over the $n$ goal facts yields a solution for $\hlprob$ by definition of GI.
\end{proof}

The bound \Cref{eq:finite-cover} is exponential in the horizon $C$ and the maximum schema arity $N$, mirroring the worst-case complexity of generalised planning over relational domains.
In practice, the number of demonstrations required is far smaller, since most singleton-goal subproblems share regression-extracted rules that lift to many equivalence classes simultaneously, and our experiments in \Cref{sec:experiments} confirm that a modest training set already suffices for $\hlp$ to generalise to arbitrary numbers of objects.

% !TEX root =mainfile.tex

\section{Additional Experimental Details}
\subsection{GNN Hyperparameters}\label{app:hyperparams}
We train all GNN models using the Adam optimiser~\citep{kingma.ba.2015} for 200 iterations with an initial learning rate of $10^{-3}$, batch size of 128, and a cosine annealing learning rate scheduler~\citep{loshchilov.hutter.2017}.
Each GNN has a hidden and output dimension of 64 and 2 message passing layers.

\subsection{Baselines}\label{app:methods}
We provide more details on the baselines of our experiments.

\paragraph{\vla{}}
\vla{} is the 450 million parameter open-source SmolVLA \citep{shukor.etal.2025}, accessed from \url{https://huggingface.co/lerobot/smolvla_base}

\paragraph{\vlamw{}}
\vlamw{} is \vla{} that is finetuned on 2500 episodes of the original MetaWorld benchmark suite (50 episodes from each of the 50 tasks), accessed from \url{https://huggingface.co/jadechoghari/smolvla_metaworld}

\paragraph{\purenn{}}
\purenn{} is a special case of \pddlnn{} in which all HL information is not encoded by assuming that $\hls = \hlg = \emptyset$ and by not encoding an HL action node.

\paragraph{\pddlnn{}}
\pddlnn{} is a similar GNN architecture to the one in \Cref{ssec:ll-policy} which encodes all HL state information including all facts and all other objects but no HL action.
We can view \pddlnn{} as a purely end-to-end approach for handling bilevel planning that aims to learn the HL policy in place of using a separately learned, symbolic HL policy.
We sketch the heterogeneous GNN architecture for \pddlnn{} as follows.

The inputs are $\lls, \hlg$ and $\hls = \labelling(\lls)$ and encoded as a graph with a global node with initial embedding $\h_{\globalNode}$, no action node, object nodes with embeddings $\seta{\h_{o} \mid o \in \objects}$ and fact nodes with embeddings $\seta{\h_{f} \mid f \in \hls \cup \hlg, \arity(f) \geq 2}$.
The initial embeddings are now given by
\begin{itemize}
    \item $\h_{\globalNode} = \x_e \concat \sum_{p() \in \hls} \onehot{p}{\abs{\predicates}} \concat \sum_{p() \in \hlg} \onehot{p}{\abs{\predicates}}$
    \item $\h_{o} = \x_{o} \concat \sum_{p(o) \in \hls} \onehot{p}{\abs{\predicates}} \concat \sum_{p(o) \in \hlg} \onehot{p}{\abs{\predicates}}$ and
    \item $\h_{f} = [f \in \hls, f \in \hlg] \concat \onehot{p}{\abs{\predicates}}$ where $p$ is the predicate of $f$.
\end{itemize}

The GNN algorithm for \pddlnn{} then consists of the following three operations.
All $\W$ symbols denote learnable weights.
\begin{enumerate}[label=(\alph*)]
    \item The initial encodings are {embedded} $\h_g^{(0)} = \W_g^{(0)} \h_g$, 
    $\h_{o}^{(0)} = \W_o^{(0)} \h_{o}$, and
    $\h_{f}^{(0)} = \W_f^{(0)} \h_{f}$.
    \item Next, $L$ rounds of {heterogeneous message passing\footnote{\url{https://pytorch-geometric.readthedocs.io/en/latest/tutorial/heterogeneous.html}}} are performed.
    Let $\gena{U, e, V}$ denote a directed edge type from node type $U$ to node type $V$ via relation $e$.
    We overload the notations $U$ and $\gen{U, e, V}$ to also mean a set that contains the nodes of node type $U$, and pairs of nodes $\gena{u, v}$ which have an edge of type $\gena{U, e, V}$, respectively.
    For each fact node $f = p(o_1, \ldots, o_m) \in \hls \cup \hlg$ with $m \geq 2$, we include bidirectional edge types $\gena{F, \mathrm{fct}_j, O}$ and $\gena{O, \mathrm{fct}_j^{-1}, F}$ for each $j = 1, \ldots, m$.
    Node embeddings are updated per edge type by aggregating neighbour messages and the global node to give the updated embedding 
    \begin{align*}
        \h_v^{(l+1)} = \sigma\!\left(\W_v \h_v^{(l)} + \W_{\globalNode}\h_{\globalNode}^{(l)} + \sum_{\gena{U,e,V}} \max_{\gena{u, v} \in \gena{U,e,V}} \W_{\gen{U,e,V}}^{(l)} \h_{u}^{(l)}\right)
    \end{align*}
    for nonlinearity $\sigma$.
    We then update the global node
    \begin{align*}
    \h_{\globalNode}^{(l+1)} = \sigma\!\left(
        \sum_{U} \max_{u \in U} \W_{U} \h_{u}^{(l+1)}
    \right).
    \end{align*}
    \item Finally, a feed-forward {readout} layer is applied to the sum of all final embeddings to predict the LL action $\lla$.
\end{enumerate}

\paragraph{\detplan{}}
\Cref{alg:detplan} illustrates the \detplan{} baseline in more detail.
It firstly computes an HL plan $a_1, \ldots, a_n$ and aims to naively and faithfully follow it throughout the simulation.
If the precondition of the current HL action, which is tracked by the index $i$, is satisfied in the current HL state, then we keep it.
Otherwise we check if the next action exists and its precondition is satisfied.
Otherwise, we determine that the HL plan is broken and terminate with failure as no HL action can be computed.
% !TEX root =mainfile.tex

\begin{algorithm}[ht]
\algorithmSize
  \DontPrintSemicolon
  \LinesNumbered
  \RestyleAlgo{ruled}
  \SetKwFor{While}{while}{}{}
    \caption{\detplan{} Baseline}\label{alg:detplan}
    \KwInput{Problem $\problem$ domain $\domain$, labelling function $\labelling$, HL goal $\hlg$, and LL policy $\llp$.}
    $\lls \gets \init$ \;
    $\hls \gets \labelling(\lls)$ \; 
    $a_1, \ldots, a_n \la \findPlan(\hls, \hlg)$ \; \label{line:detplan:plan}
    $i \la 1$ \;
  \While{$\hls \notin \hlg$ and $i \leq n$}{
    $\hls \gets \labelling(\lls)$ \;
    \If{$\pre(a_i) \subseteq \hls$}{\label{line:detplan:good}
        $\hla \gets a_i$ \; 
    }
    \ElseIf{$\pre(a_i) \not\subseteq \hls$ and $i + 1 \leq n$ and $\pre(a_{i+1})\subseteq \hls$}{\label{line:detplan:next}
        $i \gets i + 1$ \; 
        $\hla \gets a_i$ \; 
    }
    \Else{\label{line:detplan:fail}
        \Return{Failure}\; 
    }

    $\lla \gets \llp(\lls, \hla, \hlg)$ \; 
    $\lls \gets \exec(\lls, \lla)$ \; 
  }
\end{algorithm}

\paragraph{\detreplan{}}
\Cref{alg:detreplan} illustrates the \detreplan{} baseline in more detail.
It is the same as \Cref{alg:detplan} except in how it handles the failure case by replanning.
It computes a new HL plan and resets the index $i$ and length of the plan $n$.
% !TEX root =mainfile.tex

\begin{algorithm}[ht]
\algorithmSize
  \DontPrintSemicolon
  \LinesNumbered
  \RestyleAlgo{ruled}
  \SetKwFor{While}{while}{}{}
    \caption{\detreplan{} Baseline}\label{alg:detreplan}
    \KwInput{Problem $\problem$ domain $\domain$, labelling function $\labelling$, HL goal $\hlg$, and LL policy $\llp$.}
    $\lls \gets \init$ \;
    $\hls \gets \labelling(\lls)$ \; 
    $a_1, \ldots, a_n \la \findPlan(\hls, \hlg)$ \; \label{line:detreplan:plan}
    $i \la 1$ \;
  \While{$\hls \notin \hlg$ and $i \leq n$}{
    $\hls \gets \labelling(\lls)$ \;
    \If{$\pre(a_i) \subseteq \hls$}{\label{line:detreplan:good}
        $\hla \gets a_i$ \; 
    }
    \ElseIf{$\pre(a_i) \not\subseteq \hls$ and $i + 1 \leq n$ and $\pre(a_{i+1})\subseteq \hls$}{\label{line:detreplan:next}
        $i \gets i + 1$ \; 
        $\hla \gets a_i$ \; 
    }
    \Else{\label{line:detreplan:fail}
        $a_1, \ldots, a_n \la \findPlan(\hls, \hlg)$ \;
        $i \la 1$ \;
        \textbf{continue} \;
    }

    $\lla \gets \llp(\lls, \hla, \hlg)$ \; 
    $\lls \gets \exec(\lls, \lla)$ \; 
  }
\end{algorithm}

\paragraph{\regplan{}}
% !TEX root =mainfile.tex

\begin{algorithm}[ht]
\algorithmSize
  \DontPrintSemicolon
  \LinesNumbered
  \RestyleAlgo{ruled}
  \SetKwFor{While}{while}{}{}
    \caption{\ndtplan{} Baseline}\label{alg:ndtplan}
    \KwInput{Problem $\problem$ domain $\domain$, labelling function $\labelling$, HL goal $\hlg$, and LL policy $\llp$.}
    $\lls \gets \init$ \;
    $\hls \gets \labelling(\lls)$ \; 
    $\hlp \la \findPolicy(\hls, \hlg)$ \; \label{line:ndtplan:policy}
  \While{$\hls \notin \hlg$}{
    $\hls \gets \labelling(\lls)$ \;
    $\hla \gets \hlp(\hls, \hlg)$ \;
    \If{$\hla = \bot$}{\label{line:ndtplan:bad}
        \Return{Failure}\; 
    }

    $\lla \gets \llp(\lls, \hla, \hlg)$ \; 
    $\lls \gets \exec(\lls, \lla)$ \; 
  }
\end{algorithm}

\Cref{alg:ndtplan} illustrates the \ndtplan{} baseline in more detail.
It firstly aims to compute an HL policy specific to the problem and tries to execute it.
The major difference is that the HL policy may return nothing for a given state, in which case we terminate with failure.

\paragraph{\regreplan{}}
\Cref{alg:ndtreplan} illustrates the \ndtreplan{} baseline in more detail.
It is the same as \Cref{alg:ndtplan} except it now aims to replan by recomputing an HL policy for the failure case.
% !TEX root =mainfile.tex

\begin{algorithm}[ht]
\algorithmSize
  \DontPrintSemicolon
  \LinesNumbered
  \RestyleAlgo{ruled}
  \SetKwFor{While}{while}{}{}
    \caption{\ndtreplan{} Baseline}\label{alg:ndtreplan}
    \KwInput{Problem $\problem$ domain $\domain$, labelling function $\labelling$, HL goal $\hlg$, and LL policy $\llp$.}
    $\lls \gets \init$ \;
    $\hls \gets \labelling(\lls)$ \; 
    $\hlp \la \findPolicy(\hls, \hlg)$ \; \label{line:ndtreplan:policy}
  \While{$\hls \notin \hlg$}{
    $\hls \gets \labelling(\lls)$ \;
    $\hla \gets \hlp(\hls, \hlg)$ \;
    \If{$\hla = \bot$}{\label{line:ndtreplan:bad}
        $\hlp \la \findPolicy(\hls, \hlg)$ \;
        \textbf{continue}\;
    }

    $\lla \gets \llp(\lls, \hla, \hlg)$ \; 
    $\lls \gets \exec(\lls, \lla)$ \; 
  }
\end{algorithm}

\subsection{Benchmarks}\label{app:benchmarks}
The benchmarks compositionally extend the MetaWorld \citep{yu.etal.2019,mclean.etal.2025} benchmarks which consist of 3D MuJoCo simulations of a robot arm.
The following are descriptions of the environment tasks and also serve as the natural language inputs to the VLA baseline.
The environments are illustrated in \Cref{fig:envs}.

\begin{figure}
    \newcommand{\wwww}{0.24\textwidth}
    \centering
    \begin{subfigure}[b]{\wwww}
        \includegraphics[width=\textwidth]{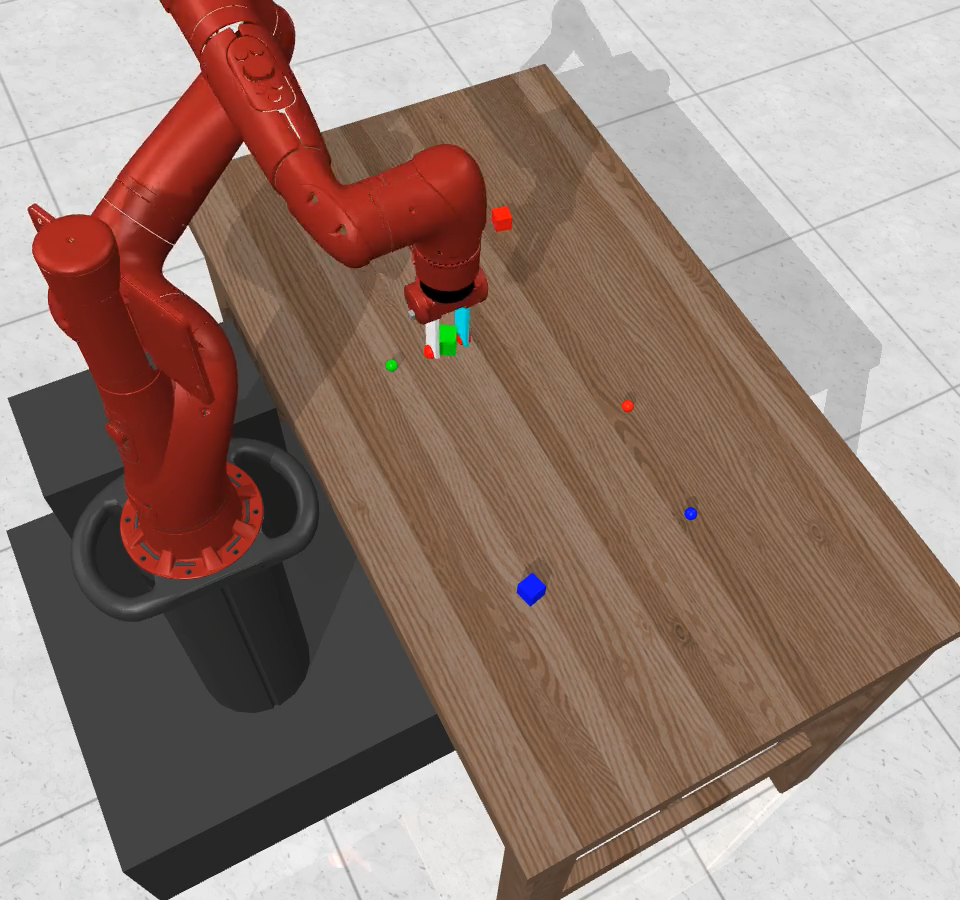}
        \subcaption{Blocks}
    \end{subfigure}
    \begin{subfigure}[b]{\wwww}
        \includegraphics[width=\textwidth]{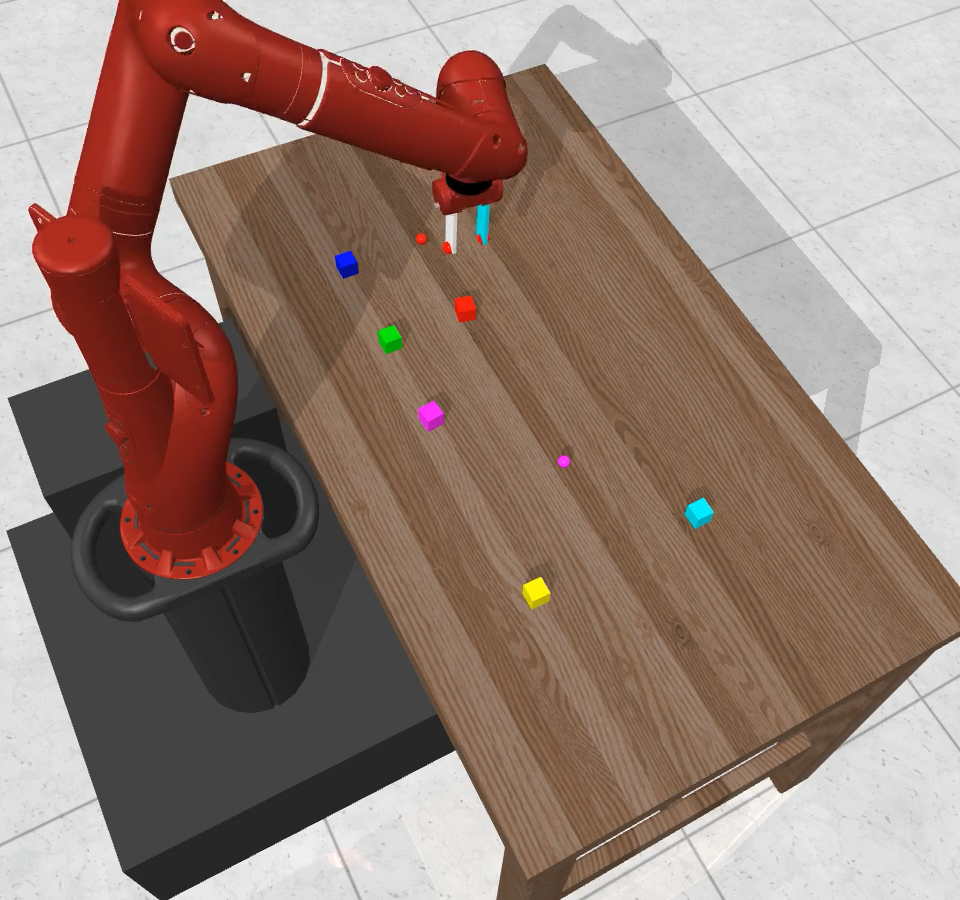}
        \subcaption{Factory}
    \end{subfigure}
    \begin{subfigure}[b]{\wwww}
        \includegraphics[width=\textwidth]{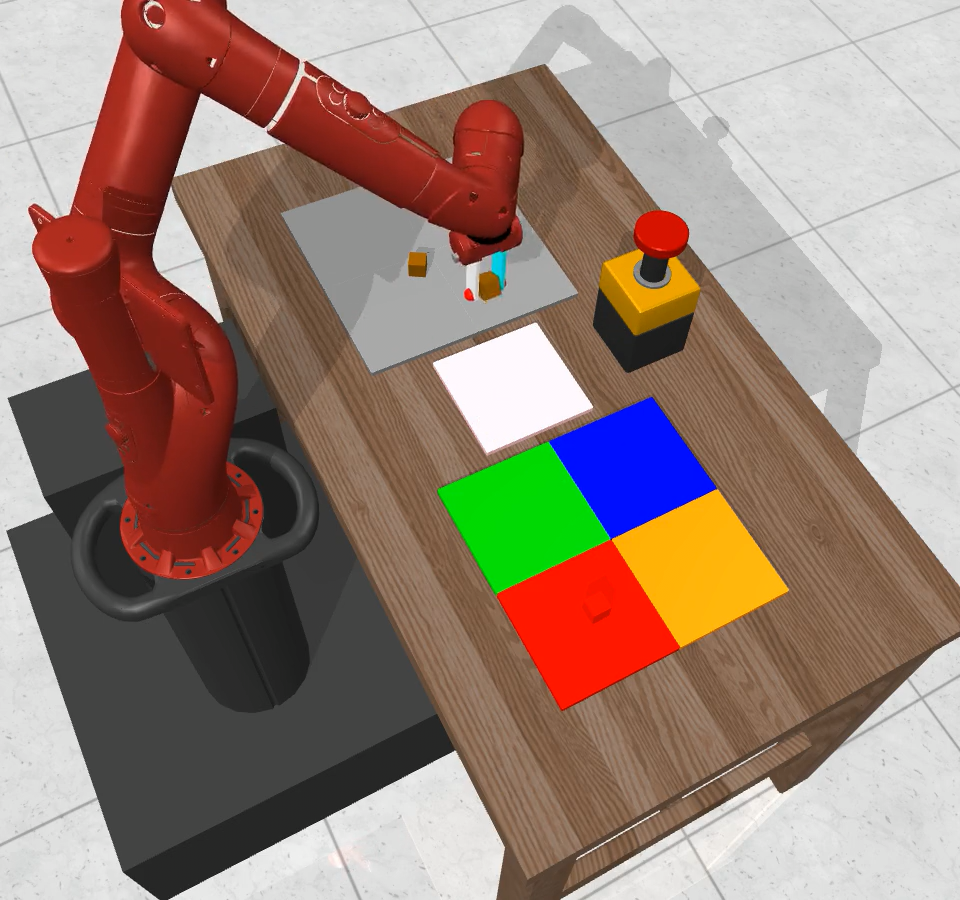}
        \subcaption{Colour}
    \end{subfigure}
    \begin{subfigure}[b]{\wwww}
        \includegraphics[width=\textwidth]{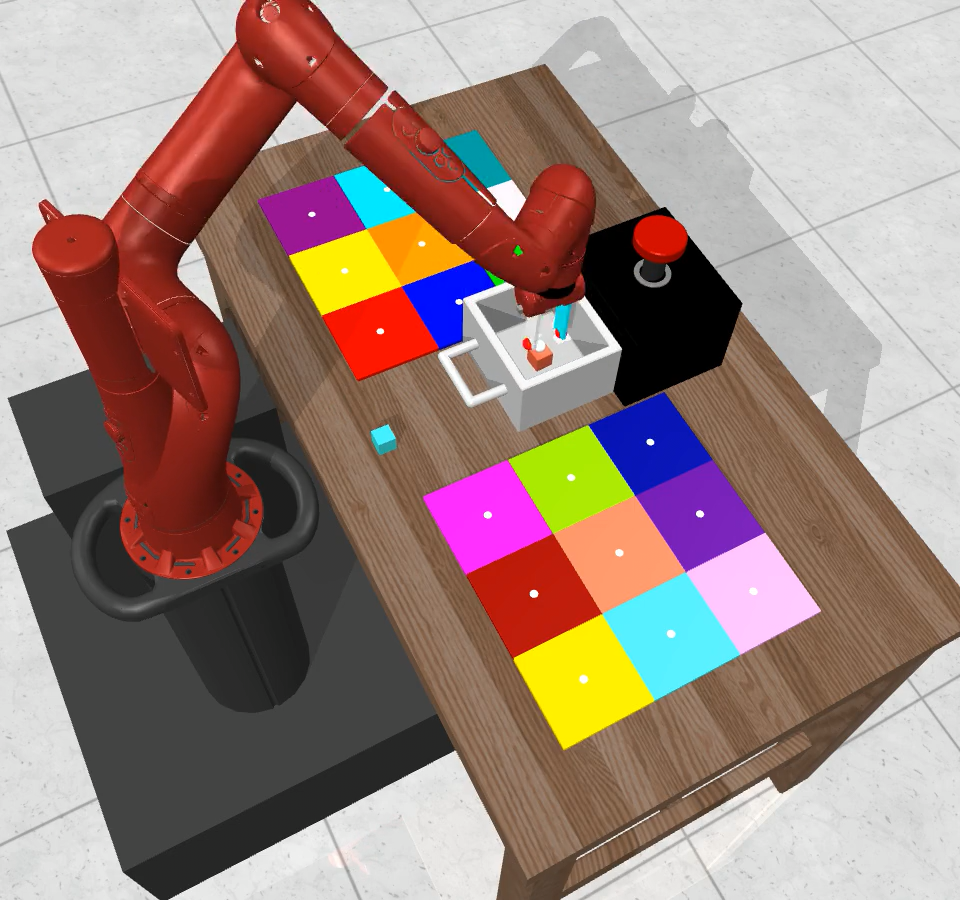}
        \subcaption{Gacha}
    \end{subfigure}
    \caption{Visualisations of benchmark environments.}\label{fig:envs}
\end{figure}

\paragraph{\blocks{}}
Place the $n$ blocks to goal locations of the same colour.

\paragraph{\blocksNoisy{}}
Place the $n$ blocks to goal locations of the same colour. Blocks that are currently at their goal location may teleport to an unoccupied location at any time.

\paragraph{\factory{}}
Place the $n$ blocks to goal locations of the same colour. Every time each of the original blocks is placed, a new block is introduced with a new goal location.

\paragraph{\factoryNoisy{}}
Place the $n$ blocks to goal locations of the same colour. Every time each of the original blocks is placed, a new block is introduced with a new goal location. Blocks that are currently at their goal location may also teleport to an unoccupied location at any time.

\paragraph{\colour{}}
Place the $n$ brown blocks onto the pink platform, which assigns blocks a random colour when activated. The goal is to colour each of the brown blocks and then place them on a tray with the same colour.

\paragraph{\colourNoisy{}}
Place the $n$ brown blocks onto the pink platform, which assigns blocks a random colour when activated. The goal is to colour each of the brown blocks and then place them on a tray with the same colour. Blocks that are currently at their goal location may also teleport to an unoccupied location at any time.

\paragraph{\gacha{}}
This environment consists of $n$ target colours, a gallery of trays with different colours, and a \textit{gacha} box which when its lid is closed and activated produces a block of a random colour.
The goal is to match each target colour by placing a block of that colour on its corresponding tray.

\paragraph{\gachaNoisy{}}
This environment consists of $n$ target colours, a gallery of trays with different colours, and a \textit{gacha} box which when its lid is closed and activated produces a block of a random colour.
The goal is to match each target colour by placing a block of that colour on its corresponding tray.
Blocks that are currently at their goal location may teleport back into the gacha box whenever it is unoccupied.

\clearpage
\subsection{Visualising Learned HL Policies}\label{app:interpret}
We display the learned HL policies for each of the environments which can generalise to arbitrary numbers of objects.
We refer to \Cref{ex:pick-place-policy} for interpreting the form of the rules.
We automatically generate LLM interpretations of the policies via Claude Sonnet 4.6.

\newcommand{\policySize}{\footnotesize}

\paragraph{\blocks{} and \blocksNoisy{}}
{
\policySize
\begin{align*}
    1: \exists x, l. \; \mi{holding}(x) \ruleAnd \mi{clear}(l)
    \ruleAnd \underline{\mi{at}(x, l)} 
    &\mapsto \mi{place}(x, l) \\
    2: \exists x, l. \; \mi{clear}(x) \ruleAnd \mi{gripperFree}()
    \ruleAnd \underline{\mi{at}(x, l)} 
    &\mapsto \mi{pick}(x)
\end{align*}
}

\underline{LLM-interpretation}:
The policy has two rules, applied by priority.
Rule 1 says: if the gripper is holding block $x$ and its goal location $l$ is clear, place $x$ at $l$.
Rule 2 says: if block $x$ is clear (nothing on top), the gripper is free, and $x$ has a goal location, pick up $x$.
Together these encode a simple pick-and-place loop: pick the next block with an unmet goal, then immediately place it.
The policy also handles \blocksNoisy{} reactively, since the goal predicate $\mi{at}^{\G}(x,l)$ remains true for any block teleported away from its goal, so it gets re-picked on the next step.

\paragraph{\factory{} and \factoryNoisy{}}
{
\policySize
\begin{align*}
    1: \exists x, l. \; \mi{holding}(x) \ruleAnd \mi{clear}(l)
    \ruleAnd \underline{\mi{at}(x, l)} 
    &\mapsto \mi{place}(x, l) \\
    2: \exists x, l. \; \mi{clear}(x) \ruleAnd \mi{gripperFree}()
    \ruleAnd \underline{\mi{at}(x, l)} 
    &\mapsto \mi{pick}(x)
\end{align*}
}

\underline{LLM-interpretation}:
The rules are identical to those of \blocks{}.
This is notable: even though new blocks are introduced dynamically when existing ones are placed, the policy generalises automatically because its rules are existentially quantified over objects.
Any newly introduced block $x$ with goal location $l$ immediately satisfies Rule 2 and gets picked, then placed by Rule 1, with no change to the policy itself.

\paragraph{\colour{} and \colourNoisy{}}
\newcommand{\mapperto}{\mapsto}
{
\policySize
\begin{align*}
    1: \exists c, x, t. \; \mi{colourOf}(x, c) \ruleAnd \mi{colourOf}(t, c) \ruleAnd \mi{holding}(x)
    \ruleAnd \underline{\mi{achievedGoal}(x)} \\ 
    \mapperto \mi{placeGoal}(x, t, c) \\
    2: \exists c, x, t. \; \mi{colourOf}(x, c) \ruleAnd \mi{colourOf}(t, c) \ruleAnd \mi{gripperFree}() \ruleAnd \mi{on}(x, t)
    \ruleAnd \underline{\mi{achievedGoal}(x)} \\ 
    \mapperto \mi{pick}(x) \\
    2: \exists c, x, t. \; \mi{colourOf}(x, c) \ruleAnd \mi{colourOf}(t, c) \ruleAnd \mi{gripperFree}()
    \ruleAnd \underline{\mi{achievedGoal}(x)} \\ 
    \mapperto \mi{pick}(x) \\
    3: \exists b, c, x, t, t'. \; \mi{colourOf}(t, c) \ruleAnd \mi{colourer}(t') \ruleAnd \mi{gripperFree}() \ruleAnd \mi{on}(x, t') \ruleAnd \mi{unknown}(x)
    \ruleAnd \underline{\mi{achievedGoal}(x)} \\ 
    \mapperto \mi{open}(x, b, t') \\
    4: \exists c, x, t, t'. \; \mi{colourOf}(t, c) \ruleAnd \mi{colourerFree}() \ruleAnd \mi{colourer}(t') \ruleAnd \mi{holding}(x) \ruleAnd \mi{unknown}(x)
    \ruleAnd \underline{\mi{achievedGoal}(x)} \\ 
    \mapperto \mi{place}(x, t') \\
    5: \exists c, x, t, t'. \; \mi{colourOf}(t, c) \ruleAnd \mi{colourer}(t') \ruleAnd \mi{gripperFree}() \ruleAnd \mi{on}(x, t') \ruleAnd \mi{unknown}(x)
    \ruleAnd \underline{\mi{achievedGoal}(x)} \\ 
    \mapperto \mi{pick}(x) \\
    5: \exists c, x, t, t'. \; \mi{colourOf}(t, c) \ruleAnd \mi{colourerFree}() \ruleAnd \mi{colourer}(t') \ruleAnd \mi{gripperFree}() \ruleAnd \mi{unknown}(x)
    \ruleAnd \underline{\mi{achievedGoal}(x)} \\ 
    \mapperto \mi{pick}(x)
\end{align*}
}

\underline{LLM-interpretation}:
The task requires colouring brown (unknown) blocks on a pink platform before placing them on matching trays, and the policy decomposes the workflow into five priority levels.
Rule 1: if holding a block $x$ whose colour $c$ matches a tray $t$ of the same colour, place $x$ on $t$.
Rule 2 (two variants): if a coloured block $x$ already matches a tray colour, pick it up; the two variants differ only in whether $x$ is currently resting on a surface.
Rule 3: if an unknown block $x$ is sitting on the colourer $t'$, activate the colourer (\mi{open}) to assign $x$ a colour.
Rule 4: if the gripper is holding an unknown block and the colourer is free, place the block on the colourer.
Rule 5 (two variants): if a block has no colour yet, pick it up so it can be brought to the colourer.
The priority ordering encodes the full pipeline: fetch unknown block, place it on the colourer, activate the colourer, then pick up the now-coloured block and place it on the matching tray.
The same policy handles \colourNoisy{} reactively, as in \blocks{}.

\clearpage

\paragraph{\gacha{} and \gachaNoisy{}}
{
\tiny
\begin{align*}
    1: \exists c, x, t. \; \mi{colourOf}(x, c) \ruleAnd \mi{holding}(x) \ruleAnd \mi{trayColour}(t, c)
    \ruleAnd \underline{\mi{achievedGoal}(c)} \\ 
    \mapperto \mi{placeGoal}(x, t, c) \\
    2: \exists c, x, d, t. \; \mi{colourOf}(x, c) \ruleAnd \mi{gripperFree}() \ruleAnd \mi{in}(x, d) \ruleAnd \mi{opened}(d) \ruleAnd \mi{trayColour}(t, c)
    \ruleAnd \underline{\mi{achievedGoal}(c)} \\ 
    \mapperto \mi{pick}(x) \\
    3: \exists c, x, d, t. \; \mi{colourOf}(x, c) \ruleAnd \mi{gripperFree}() \ruleAnd \mi{in}(x, d) \ruleAnd \mi{trayColour}(t, c)
    \ruleAnd \underline{\mi{achievedGoal}(c)} \\ 
    \mapperto \mi{open}(d) \\
    4: \exists b, c, d, t. \; \mi{clear}(d) \ruleAnd \mi{closed}(d) \ruleAnd \mi{gripperFree}() \ruleAnd \mi{trayColour}(t, c)
    \ruleAnd \underline{\mi{achievedGoal}(c)} \\ 
    \mapperto \mi{roll}(d, b) \\
    5: \exists c, d, t. \; \mi{clear}(d) \ruleAnd \mi{gripperFree}() \ruleAnd \mi{opened}(d) \ruleAnd \mi{trayColour}(t, c)
    \ruleAnd \underline{\mi{achievedGoal}(c)} \\ 
    \mapperto \mi{close}(d) \\
    6: \exists c, c', x, d, t, t'. \; \mi{clear}(d) \ruleAnd \mi{colourOf}(x, c) \ruleAnd \mi{holding}(x) \ruleAnd \mi{opened}(d) \ruleAnd \mi{trayColour}(t, c) \ruleAnd \mi{trayColour}(t', c')
    \ruleAnd \underline{\mi{achievedGoal}(c')} \\
    \mapperto \mi{placeGoal}(x, t, c) \\
    7: \exists c, c', x, d, t, t'. \; \mi{colourOf}(x, c) \ruleAnd \mi{gripperFree}() \ruleAnd \mi{in}(x, d) \ruleAnd \mi{opened}(d) \ruleAnd \mi{trayColour}(t, c) \ruleAnd \mi{trayColour}(t', c')
    \ruleAnd \underline{\mi{achievedGoal}(c')} \\
    \mapperto \mi{pick}(x) \\
    7: \exists c, c', x, d, t, t'. \; \mi{colourOf}(x, c) \ruleAnd \mi{gripperFree}() \ruleAnd \mi{in}(x, d) \ruleAnd \mi{opened}(d) \ruleAnd \mi{trayColour}(t, c') \ruleAnd \mi{trayColour}(t', c)
    \ruleAnd \underline{\mi{achievedGoal}(c')} \\
    \mapperto \mi{pick}(x) \\
    8: \exists c, c', x, d, t, t'. \; \mi{colourOf}(x, c) \ruleAnd \mi{gripperFree}() \ruleAnd \mi{in}(x, d) \ruleAnd \mi{trayColour}(t, c') \ruleAnd \mi{trayColour}(t', c)
    \ruleAnd \underline{\mi{achievedGoal}(c')} \\
    \mapperto \mi{open}(d) \\
    8: \exists c, c', x, d, t, t'. \; \mi{colourOf}(x, c) \ruleAnd \mi{gripperFree}() \ruleAnd \mi{in}(x, d) \ruleAnd \mi{trayColour}(t, c) \ruleAnd \mi{trayColour}(t', c')
    \ruleAnd \underline{\mi{achievedGoal}(c')} \\
    \mapperto \mi{open}(d) \\
    9: \exists b, c, c', d, t, t'. \; \mi{clear}(d) \ruleAnd \mi{closed}(d) \ruleAnd \mi{gripperFree}() \ruleAnd \mi{trayColour}(t, c) \ruleAnd \mi{trayColour}(t', c')
    \ruleAnd \underline{\mi{achievedGoal}(c)} \\ 
    \mapperto \mi{roll}(d, b) \\
    9: \exists b, c, c', d, t, t'. \; \mi{clear}(d) \ruleAnd \mi{closed}(d) \ruleAnd \mi{gripperFree}() \ruleAnd \mi{trayColour}(t, c) \ruleAnd \mi{trayColour}(t', c')
    \ruleAnd \underline{\mi{achievedGoal}(c')} \\
    \mapperto \mi{roll}(d, b)
\end{align*}
}

\underline{LLM-interpretation}:
The task requires producing blocks of specific colours from a gacha box (close lid, roll, open lid, then extract block) and placing them on matching trays.
The policy has nine priority levels.
Rule 1: if holding a block $x$ of the needed colour $c$, place it on the matching tray $t$.
Rule 2: if the gacha $d$ is open and contains a block of the needed colour, pick it.
Rule 3: if the gacha contains a block of the needed colour but is closed, open it.
Rule 4: if the gacha is closed and empty, roll it to produce a new block of random colour.
Rule 5: if the gacha is open but does not contain a useful block, close it so that it can be rolled again.
Rules 6--8 cover \emph{opportunistic} actions when the currently pursued goal colour is $c'$ but a block of a different needed colour $c$ becomes available: place it on its matching tray (Rule 6), pick it from an open gacha (Rule 7, two variants differing in the binding of $c$ and $c'$ to the trays), or open the gacha to retrieve it (Rule 8, two variants).
Rule 9 (two variants) rolls the gacha when at least two unachieved goal colours remain and the gacha is empty.
The policy is more complex than \blocks{} because it must reason about both the stochastic production mechanism (rolls produce random colours) and opportunity costs across multiple simultaneous colour goals; Rules 6--9 in particular ensure that lucky rolls are never wasted, even when they do not match the colour the agent is currently chasing.

%%%%%%%%%%%%%%%%%%%%%%%%%%%%%%%%%%%%%%%%%%%%%%%%%%%%%%%%%%%%

\end{document}